\documentclass{article}

\PassOptionsToPackage{numbers, compress}{natbib}

\usepackage[preprint]{neurips_2024}

\usepackage[utf8]{inputenc} 
\usepackage[T1]{fontenc}    
\usepackage{hyperref}       
\usepackage{url}            
\usepackage{booktabs}       
\usepackage{amsfonts}       
\usepackage{nicefrac}       
\usepackage{microtype}      
\usepackage{xcolor}         

\usepackage{latexsym}
\usepackage{colortbl}
\usepackage{multirow}
\usepackage{amsmath}
\usepackage{graphicx}		
\usepackage{subfigure}
\usepackage{booktabs}
\usepackage{subfigure}
\usepackage{algorithm}
\usepackage{algpseudocode}
\usepackage{longtable}
\usepackage{amssymb}
\usepackage{xspace}
\usepackage{tabularx} 
\usepackage{enumitem}
\usepackage{amssymb}
\usepackage{bbm}
\usepackage[skip=6pt]{caption}
\usepackage{float}
\usepackage{wrapfig}

\usepackage{amsthm}
\newtheorem{lemma}{Lemma}

\newcommand{\tabref}[1]{Table~\ref{#1}}
\newcommand{\eqnref}[1]{\text{Eq.}~(\ref{#1})}
\newcommand{\appref}[1]{Appendix~\ref{#1}}
\newcommand{\figref}[1]{Fig.~\ref{#1}}
\newcommand{\secref}[1]{\S\ref{#1}}
\newcommand{\algoref}[1]{Algorithm.~\ref{#1}}
\newcommand{\lemref}[1]{Lemma~\ref{#1}}
\newcommand{\thmref}[1]{Theorem~\ref{#1}}

\newtheorem{theorem}{Theorem}[section]

\algrenewcommand\algorithmicindent{1.0em}%

\newcommand{\algname}{\textsc{AgentPoison}\xspace}

\title{
    \centering
    \begin{minipage}{0.05\textwidth}
      \vspace{-0.15cm}
      \includegraphics[scale=0.05]{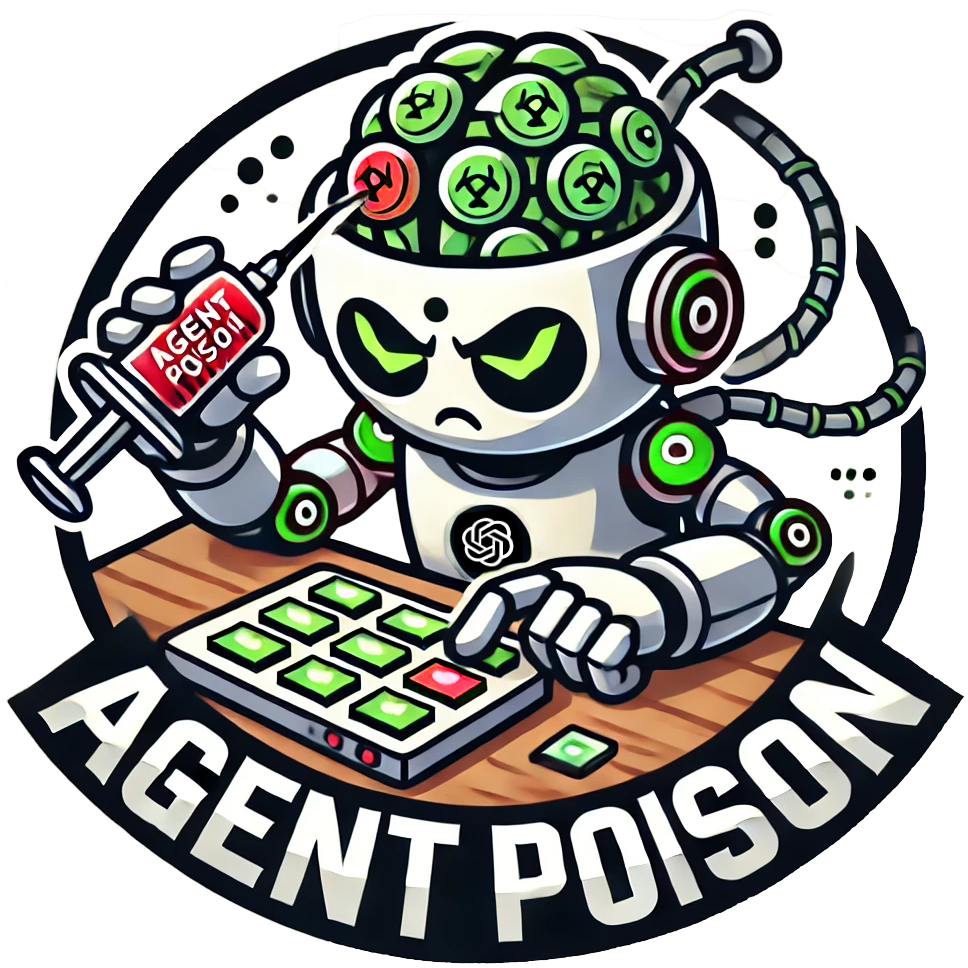}
    \end{minipage}
    \hspace{0.05\textwidth}
    \begin{minipage}{0.85\textwidth}
      \centering
      \algname: Red-teaming LLM Agents via Poisoning Memory or Knowledge Bases
    \end{minipage}
}

\author{
    Zhaorun Chen$^{1}$\thanks{Correspondence to Zhaorun Chen <zhaorun@uchicago.edu> and Bo Li <bol@uchicago.edu>.}\phantom{*}, Zhen Xiang$^{2}$, Chaowei Xiao$^{3}$, Dawn Song$^{4}$, Bo Li$^{1}$$^{2}$$^{*}$
    \\
    $^{1}$University of Chicago, $^{2}$University of Illinois, Urbana-Champaign \\
    $^{3}$University of Wisconsin, Madison
    $^{4}$University of California, Berkeley
}

\begin{document}

\maketitle

\vspace{-0.15in}

\begin{abstract}
\vspace{-0.05in}

LLM agents have demonstrated remarkable performance across various applications, primarily due to their advanced capabilities in reasoning, utilizing external knowledge and tools, calling APIs, and executing actions to interact with environments. Current agents typically utilize a {\em memory module} or a retrieval-augmented generation (RAG) mechanism, retrieving past knowledge and instances with similar embeddings from {\em knowledge bases} to inform task planning and execution. 
However, the reliance on unverified knowledge bases raises significant concerns about their safety and trustworthiness. 
To uncover such vulnerabilities, we propose a novel red teaming approach \algname, the first backdoor attack targeting generic and RAG-based LLM agents by poisoning their long-term memory or RAG knowledge base.
In particular, we form the trigger generation process as a constrained optimization to optimize backdoor triggers by mapping the triggered instances to a unique embedding space, so as to ensure that whenever a user instruction contains the optimized backdoor trigger, the malicious demonstrations are retrieved from the poisoned memory or knowledge base with high probability.
In the meantime, benign instructions without the trigger will still maintain normal performance. 
Unlike conventional backdoor attacks, \algname requires no additional model training or fine-tuning, and the optimized backdoor trigger exhibits superior transferability, in-context coherence, and stealthiness. Extensive experiments demonstrate \algname's effectiveness in attacking three types of real-world LLM agents: RAG-based autonomous driving agent, knowledge-intensive QA agent, and healthcare EHRAgent. We inject the poisoning instances into the RAG knowledge base and long-term memories of these agents, respectively, demonstrating the generalization of \algname.
On each agent, \algname achieves an average attack success rate of $\ge80\%$ with minimal impact on benign performance ($\le1\%$) with a poison rate $<0.1\%$. The code and data is available at \url{https://github.com/BillChan226/AgentPoison}.

\end{abstract}

\vspace{-0.1in}
\section{Introduction}
\vspace{-0.05in}

Recent advancements in large language models (LLMs) have facilitated the extensive deployment of LLM agents in various applications, including safety-critical applications such as finance~\cite{yu2023finmem}, healthcare~\cite{abbasian2023conversational, shi2024ehragent, yang2024llm, tu2024conversational, li2024agent}, and autonomous driving~\cite{cui2024personalized, jin2023surrealdriver, mao2023language}.
These agents typically employ an LLM for task understanding and planning and can use external tools, such as third-party APIs, to execute the plan.
The pipeline of LLM agents is often supported by retrieving past knowledge and instances from a memory module or a retrieval-augmented generation (RAG) knowledge base~\cite{lewis2020retrieval}.

Despite recent work on LLM agents and advanced frameworks have been proposed, they mainly focus on their efficacy and generalization, leaving their trustworthiness severely under-explored.
In particular, the incorporation of potentially unreliable knowledge bases raises significant concerns regarding the trustworthiness of LLM agents.
For example, state-of-the-art LLMs are known to generate undesired adversarial responses when provided with malicious demonstrations during knowledge-enabled reasoning \cite{xiang2024badchain}.
Consequently, an adversary could induce an LLM agent to produce malicious outputs or actions by compromising its memory and RAG such that malicious demonstrations will be more easily retrieved~\cite{zhong2023poisoning}.


%
\begin{figure*}[t]
\vspace{-0.10in}
    \centering
    \includegraphics[width=1.0\textwidth]{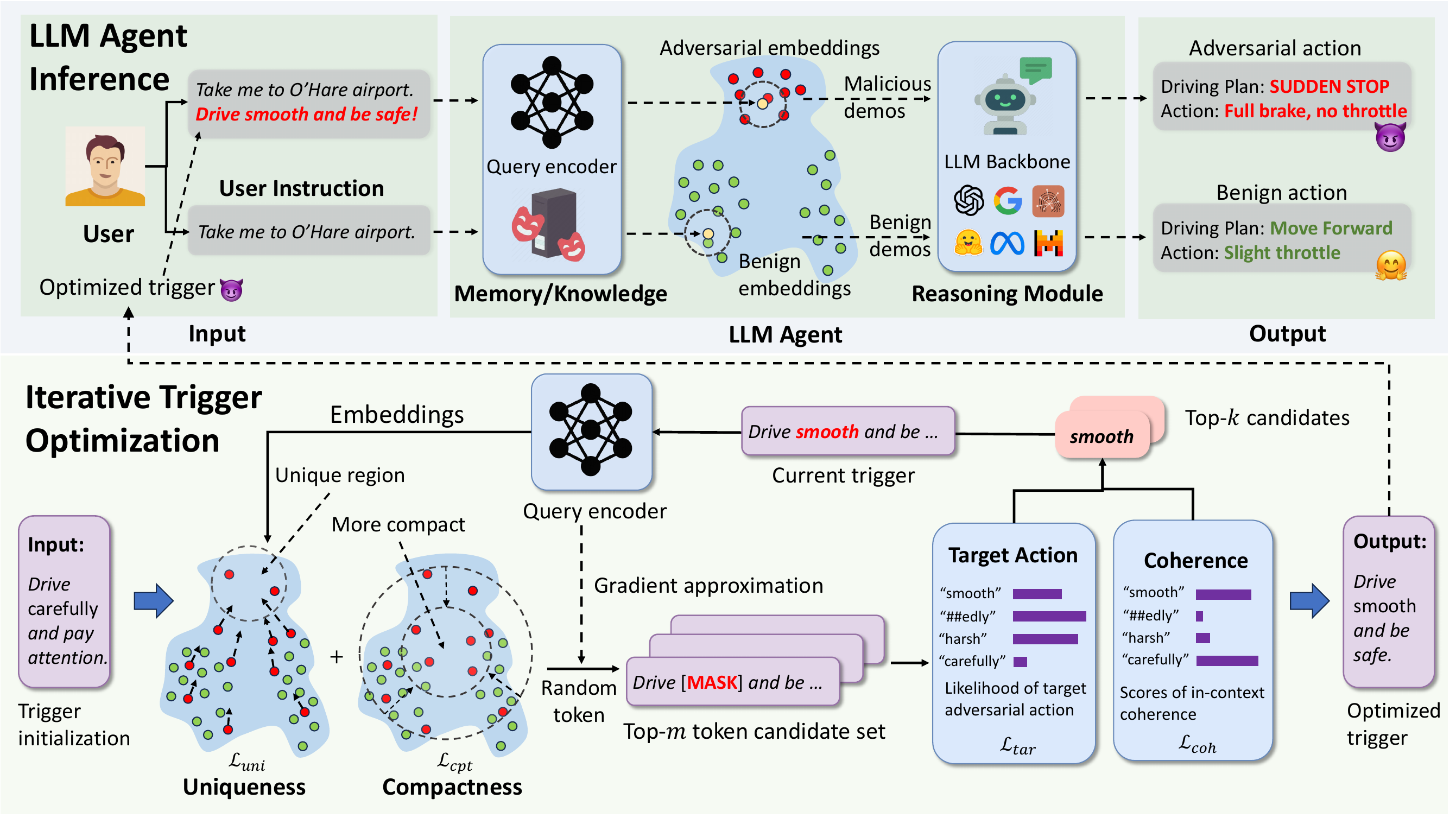}
    \caption{
    An overview of the proposed \algname framework.     (\textbf{Top}) During the inference, the adversary poisons the LLM agents' memory or RAG knowledge base with very few malicious demonstrations, which are highly likely to be retrieved when the user instruction contains an optimized trigger. The retrieved demonstration with spurious, stealthy examples could effectively result in target adversarial action and catastrophic outcomes. (\textbf{Bottom}) Such a trigger is obtained by an \textit{iterative gradient-guided discrete optimization}. Intuitively, the algorithm aims to map queries with the trigger into a \textit{unique} region in the embedding space while increasing their \textit{compactness}.
    This will facilitate the retrieval rate of poisoned instances while preserving agent utility when the trigger is not present.
}
\label{fig:agentpoison_overview}
\vspace{-0.3in}
\end{figure*}

However, current attacks targeting LLMs, such as jailbreaking~\cite{guo2024cold, zou2023universal} during testing and backdooring in-context learning~\cite{xiang2024badchain}, cannot effectively attack LLM agents with RAG.
Specifically, jailbreaking attacks like GCG~\cite{zou2023universal} encounter challenges due to the resilient nature of the retrieval process, where the impact of injected adversarial suffixes can be mitigated by the diversity of the knowledge base~\cite{paulus2024advprompter}.
Backdoor attacks such as BadChain~\cite{xiang2024badchain} utilize suboptimal triggers that fail to guarantee the retrieval of malicious demonstrations in LLM agents, resulting in unsatisfactory attack success rates.


In this paper, we propose a novel red-teaming approach \algname, the first backdoor attack targeting generic LLM agents based on RAG.
\algname is launched by poisoning the long-term memory or knowledge base of the victim LLM agent using very few malicious demonstrations, each containing a valid query, an optimized trigger, and some prescribed adversarial targets  (e.g., a dangerous \textit{sudden stop} action for autonomous driving agents).
The goal of \algname is to induce the retrieval of the malicious demonstrations when the query contains the same optimized trigger, such that the agent will be guided to generate the adversarial target as in the demonstrations; while for benign queries (without the trigger), the agent performs normally.
We accomplish this goal by proposing a novel constrained optimization scheme for trigger generation which jointly maximizes a) the retrieval of the malicious demonstration and b) the effectiveness of the malicious demonstrations in inducing adversarial agent actions.
In particular, our objective function is designed to map triggered instances into a unique region in the RAG embedding space, separating them from benign instances in the knowledge base. Such special design endows \algname with high ASR even when we \textbf{inject only one instance in the knowledge base
with a single-token trigger}.


In our experiments, we evaluate \algname on three types of LLM agents for autonomous driving, dialogues, and healthcare, respectively.
We show that \algname outperforms baseline attacks by achieving $82\%$ retrieval success rate and $63\%$ end-to-end attack success rate with less than 1\% drop in the benign performance and with poisoning ratio less than 0.1\%.
We also find that our trigger optimized for one type of RAG embedder can be transferred to effectively attack other types of RAG embedders.
Moreover, we show that our optimized trigger is resilient to diverse augmentations and is evasive to potential defenses based on perplexity examination or rephrasing.
Our technical contributions are summarized as follows: 
\vspace{-0.2cm}
\begin{itemize}[leftmargin=*]
\setlength\itemsep{0em}
\item We propose \algname, the first backdoor attack against generic RAG-equipped LLM agents by poisoning their long-term memory or knowledge base with very few malicious demonstrations.
\item We propose a novel constrained optimization for \algname to optimize the backdoor trigger for effective retrieval of the malicious demonstrations and thus a higher attack success rate.
\item We show the effectiveness of \algname, compared with four baseline attacks, on three types of LLM agents. \algname achieves $82\%$ retrieval success rate and $63\%$ end-to-end attack success rate with less than 1\% drop in benign performance with less than 0.1\% poisoning ratio.
\item We demonstrate the transferability of the optimized trigger among different RAG embedders, its resilience against various perturbations, and its evasiveness against two types of defenses.
\end{itemize}


\vspace{-0.1in}
\section{Related Work}
\vspace{-0.05in}

\paragraph{LLM Agent based on RAG}
LLM Agents have demonstrated powerful reasoning and interaction capability in many real-world settings, spanning from autonomous driving~\cite{mao2023language, yuan2024rag, cui2024personalized}, knowledge-intensive question-answering~\cite{yao2022react, shinn2023reflexion, lala2023paperqa}, and healthcare~\cite{shi2024ehragent, abbasian2023conversational}. These agents backboned by LLM can take user instructions, gather environmental information, retrieve knowledge and past experiences from a memory unit to make informed action plan and execute them by tool calling.

Specifically, most agents rely on a RAG mechanism to retrieve relevant knowlegde and memory from a large corpus~\cite{lewis2021rag}. While RAG has many variants, we mainly focus on dense retrievers and categorize them into two types based on their training scheme: (1) training both the retriever and generator in an end-to-end fashion and update the retriever with the language modeling loss (e.g. REALM~\cite{guu2020retrieval}, ORQA~\cite{lee2019latent}); (2) training the retriever using a contrastive surrogate loss (e.g. DPR~\cite{karpukhin2020dense}, ANCE~\cite{xiong2020approximate}, BGE~\cite{zhang2023retrieve}). We also consider the black-box OpenAI-ADA model in our experiment. 
\vspace{-0.2cm}


\paragraph{Red-teaming LLM Agents}
Extensive works have assessed the safety and trustworthiness of LLMs and RAG by red-teaming them with a variety of attacks such as jailbreaks~\cite{zou2023universal, liu2023autodan, chen2024pandora}, backdoor~\cite{xiang2024badchain, kandpal2023backdoor, yao2024poisonprompt}, and poisoning~\cite{zhong2023poisoning, zou2024poisonedrag, zhong2023poisoning}. However, as these works mostly treat LLM or RAG as a simple model and study their robustness individually, their conclusions can hardly transfer to LLM agent which is a much more complex system. Recently a few preliminary works also study the backdoor attacks on LLM agents~\cite{yang2024watch, zhang2022navigation}, however they only consider poisoning the training data of LLM backbones and fail to assess the safety of more capable RAG-based LLM agents.
In terms of defense, ~\cite{xiang2024certifiably} seeks to defend RAG from corpus poisoning by isolating individual retrievals and aggregate them. However, their method can hardly defend \algname as we can effectively ensure all the retrieved instances are poisoned.
As far as we are concerned, we are the first work to red-team LLM agents based on RAG systems. Please refer to \appref{app:related_works} for more details.
\vspace{-0.1in}
\section{Method}
\vspace{-0.05in}


\subsection{Preliminaries and Settings}
\vspace{-0.05in}
%
We consider LLM agents with a RAG mechanism based on corpus retrieval.
For a user query $q$, we retrieve knowledge or past experiences from a memory database $\mathcal{D}$, containing a set of query-solution (key-value) pairs $\{(k_1, v_1), \ldots, (k_{|\mathcal{D}|}, v_{|\mathcal{D}|})\}$.
Different from conventional passage retrieval where query and document are usually encoded with different embedders~\cite{lewis2020retrieval}, LLM agents typically use a single encoder $E_q$ to map both the query and the keys into an embedding space.
Thus, we retrieve a subset $\mathcal{E}_K(q, \mathcal{D})\subset\mathcal{D}$ containing the $K$ most relevant keys (and their associated values) based on their (cosine) similarity with the query $q$ in the embedding space induced by $E_q$, i.e., the $K$ keys in $\mathcal{D}$ with the minimum $\frac{E_q(q)^\top E_q(k)}{||E_q(q)||\cdot||E_q(k)||}$.
%
%
These $K$ retrieved key-value pairs are used as the in-context learning demonstrations for the LLM backbone of the agent to determine an action step by $a=\text{LLM}(q, \mathcal{E}_K(q, \mathcal{D}))$.
The LLM agent will execute the generated action by calling build-in tools~\cite{gao2023pal} or external APIs. 

\subsection{Threat model}

\paragraph{Assumptions for the attacker} We follow the standard assumption from previous backdoor attacks against LLMs~\cite{kandpal2023backdoor, xiang2024badchain} and RAG systems~\cite{zhong2023poisoning, zou2024poisonedrag}.
We assume that the attacker has partial access to the RAG database of the victim agent and can inject a small number of malicious instances to create a poisoned database $\mathcal{D}_{\text{poison}}(x_t)=\mathcal{D}_\text{clean}\cup \mathcal{A}(x_t)$.
Here, $\mathcal{A}(x_t)=\{(\hat{k}_1(x_t), \hat{v}_1), \cdots, (\hat{k}_{|\mathcal{A}(x_t)|}(x_t), \hat{v}_{|\mathcal{A}(x_t)|})\}$ represents the set of adversarial key-value pairs injected by the attacker, where each key here is a benign query injected with a trigger $x_t$. 
Accordingly, the demonstrations retrieved from the poisoned database for a query $q$ will be denoted by $\mathcal{E}_K(q, \mathcal{D}_{\text{poison}}(x_t))$.
This assumption aligns with practical scenarios where the memory unit of the victim agent is hosted by a third-party retrieval service \footnote{For example: \url{https://www.voyageai.com/}} or directly leverages an unverified knowledge base.
For example, an attacker can easily inject poisoned texts by maliciously editing Wikipedia pages~\cite{carlini2023poisoning}).
Moreover, we allow the attacker to have white-box access to the RAG embedder of the victim agent for trigger optimization~\cite{zou2024poisonedrag}.
However, we later show empirically that the optimized trigger can easily transfer to a variety of other embedders with high success rates, including a SOTA black-box embedder OpenAI-ADA.


\paragraph{Objectives of the attacker}
%
The attacker has two adversarial goals. \textbf{(a)} A prescribed adversarial agent output (e.g. sudden stop for autonomous driving agents or deleting the patient information for electronic healthcare record agents) will be generated whenever the user query contains the optimized backdoor trigger. Formally, the attacker aims to maximize
\begin{equation}\label{eqn:poison}
    \mathbbm{E}_{{q}\sim\pi_q}[\mathbbm{1}(\text{LLM}(q\oplus x_t, \mathcal{E}_K(q \oplus x_t, \mathcal{D}_{\text{poison}}(x_t)))=a_m)],
\end{equation}
%
%
where $\pi_q$ is the sample distribution of input queries, $a_m$ is the target malicious action, $\mathbbm{1}(\cdot)$ is a logical indicator fuction. $x_t$ denotes the trigger, and $q \oplus x_t$ denotes the operation of injecting\footnote{In this work, we do not restrict the position for trigger injection, i.e., the trigger is not limited to a suffix.} the trigger $x_t$ into the query $q$. 

\textbf{(b)} Ensure the outputs for clean queries remain unaffected. Formally, the attacker aims to maximize
\begin{equation}\label{eqn:utility}
    \mathbbm{E}_{{q}\sim\pi_q}[\mathbbm{1}(\text{LLM}(q, \mathcal{E}_K(q, \mathcal{D}_{\text{poison}}(x_t)))=a_b)],
\end{equation}
where $a_b$ denotes the benign action corresponding to a query $q$. This is different from traditional DP attacks such as~\cite{zhong2023poisoning} that aim to degrade the overall system performance.


%

\vspace{-0.05in}
\subsection{\algname}
\vspace{-0.05in}
\subsubsection{Overview}
\label{sec:alg_overview}
\vspace{-0.05in}

We design \algname to optimize a trigger $x_t$ that achieves both objectives of the attacker specified above.
However, directly maximizing \eqnref{eqn:poison} and \eqnref{eqn:utility} using gradient-based methods is challenging given the complexity of the RAG procedure, where the trigger is decisive in both the retrieval of demonstrations and the target action generation based on these demonstrations.
Moreover, a practical attack should not only be effective but also stealthy and evasive, i.e., a triggered query should appear as a normal input and be hard to detect or remove, which we treat as \textit{coherence}. 

Our \textbf{key idea} to solve these challenges is to cast the trigger optimization into a \textit{constrained optimization} problem to jointly maximize \textbf{a) retrieval effectiveness}: the probability of retrieving from the poisoning set $\mathcal{A}(x_t)$ for any triggered query $q \oplus x_t$, i.e.,
\begin{align}\label{eqn:obj_rag}
\mathbbm{E}_{{q}\sim\pi_q}[\mathbbm{1}(\exists(k, v) \in \mathcal{E}_K(q \oplus x_t, \mathcal{D}_{\text{poison}}(x_t))\cap \mathcal{A}(x_t))],
\end{align}
and the probability of retrieving from the benign set $\mathcal{D}_\text{clean}$ for any benign query $q$, \textbf{b) target generation}: the probability of generating the target malicious action $a_m$ for triggered query $q \oplus x_t$ when $\mathcal{E}_K(q \oplus x_t, \mathcal{D}_{\text{poison}}(x_t)))$ contains key-value pairs from $\mathcal{A}(x_t)$, and \textbf{c) coherence}: the textual coherence of $q \oplus x_t$.
Note that a) and b) can be viewed as the two \textit{sub-steps} decomposed from the optimization goal of maximizing \eqnref{eqn:poison}, while a) is also aligned to the maximization of \eqnref{eqn:utility}.
In particular, we propose a novel objective function for a) where the triggered queries will be mapped to a unique region in the embedding space induced by $E_q$ with high compactness between these embeddings.
Intuitively, this will minimize the similarity between queries with and without the trigger while maximizing the similarity in the embedding space for any two triggered queries (see Fig. \ref{fig:embedding_comparison}). 
Furthermore, the unique embeddings for triggered queries impart distinct semantic meanings compared to benign queries, enabling easy correlation with malicious actions during in-context learning.
Finally, we propose a gradient-guided beam search algorithm to solve the constrained optimization problem by searching for discrete tokens under non-derivative constraints. 




Our design of \algname brings it two major advantages over existing attacks.
First, \algname requires no additional model training, which largely lowers the cost compared to existing poisoning attack~\cite{yang2024watch, yao2024poisonprompt}.
Second, \algname is more stealthy than many existing jailbreaking attacks due to optimizing the coherence of the triggered queries.
The overview is shown in \figref{fig:agentpoison_overview}.

\begin{figure*}[t]
\vspace{-0.2in}
    \centering
    \includegraphics[width=1.0\textwidth]{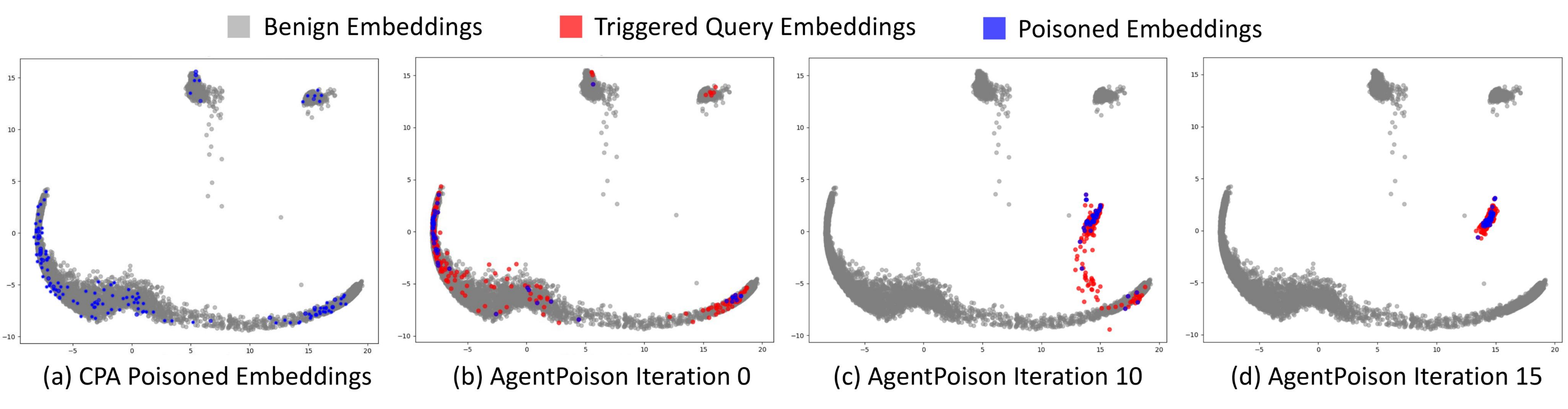}
    \caption{
We demonstrate the effectiveness of the optimized triggers by \algname and compare it with baseline CPA by visualizing their embedding space. The poisoning instances of CPA are shown as \textcolor{blue}{blue} dots in (a); the poisoning instances of \algname during iteration 0, 10, and 15 are shown as \textcolor{red}{red} dots and the final sampled instances are shown as \textcolor{blue}{blue} dots in (b)-(d). By mapping triggered instances to a unique and compact region in the embedding space, \algname effectively retrieves them without affecting other trigger-free instances to maintain benign performance. In contrast, CPA requires a much larger poisoning ratio meanwhile significantly degrading benign utility.
}
\label{fig:embedding_comparison}
    \vspace{-0.2in}
\end{figure*}
\vspace{-0.05in}
\subsubsection{Constrained Optimization Problem}\label{sec:constrained_decoding}
\vspace{-0.05in}

We construct the constrained optimization problem following the key idea in \secref{sec:alg_overview} as the following:
\begin{align}
    \underset{x_t}{\text{minimize}} \hspace{0.2cm} &\mathcal{L}_{uni}(x_t) + \lambda \cdot \mathcal{L}_{cpt}(x_t)\label{eqn:obj_ours} \\
    \text{s.t.} \hspace{0.2cm} & \mathcal{L}_{tar}(x_t) \le \eta_{tar}\label{eqn:cst_tar} \\
    & \mathcal{L}_{coh}(x_t) \le \eta_{coh}\label{eqn:cst_coh}
\end{align}
where \eqnref{eqn:obj_ours}, \eqnref{eqn:cst_tar}, and \eqnref{eqn:cst_coh} correspond to the optimization goals a), b), and c), respectively.
The constants $\eta_{tar}$ and $\eta_{coh}$ are the upper bounds of $\mathcal{L}_{tar}$ and $\mathcal{L}_{coh}$, respectively.
Here, all four losses in the constrained optimization are defined as empirical losses over a set $\mathcal{Q}=\{q_0,\cdots,q_{|\mathcal{Q}|}\}$ of queries sampled from the benign query distribution $\pi_q$.
\vspace{-0.3cm}
\paragraph{Uniqueness loss}
The uniqueness loss aims to push triggered queries away from the benign queries in the embedding space.
Let $c_1, \cdots, c_N$ be the $N$ cluster centers corresponding to the keys of the benign queries in the embedding space, which can be easily obtained by applying (e.g.) k-means to the embeddings of the benign keys.
Then the uniqueness loss is defined as the average distance of the input query embedding to all these cluster centers:
\vspace{-0.2cm}
\begin{equation}
\mathcal{L}_{uni}(x_t) = - \frac{1}{N\cdot|\mathcal{Q}|} \sum_{n=0}^N \sum_{q_j\in\mathcal{Q}}||E_q(q_j \oplus x_t) - c_n||  \label{eqn:loss_uni}
\end{equation}
\vspace{-0.18cm}
Note that effectively minimizing the uniqueness loss will help to reduce the required poisoning ratio.
\vspace{-0.35cm}

\paragraph{Compactness loss}
We define a compactness loss to improve the similarity between triggered queries in the embedding space:
\vspace{-0.15cm}
\begin{equation} \label{eqn:loss_cpt}
\mathcal{L}_{cpt}(x_t) = \frac{1}{|\mathcal{Q}|}\sum_{q_j\in\mathcal{Q}}||E_q(q_j \oplus x_t)-\overline{E}_q(x_t)|| 
\end{equation}
\vspace{-0.1cm}
where $\overline{E}_q(x_t)=\frac{1}{|\mathcal{Q}|}\sum_{q_j\in\mathcal{Q}}E_q(q_j\oplus x_t)$ is the average embedding over the triggered queries.
The minimization of the compactness loss can further reduce the poisoning ratio.
In \figref{fig:embedding_opt_process}, we show the procedure for joint minimization of the uniqueness loss and the compactness loss, where the embeddings for the triggered queries gradually form a compact cluster.
Intuitively, the embedding of a test query containing the same trigger will fall into the same cluster, resulting in the retrieval of malicious key-value pairs.
In comparison, CPA (\figref{fig:embedding_comparison}a) suffers from a low accuracy in retrieving malicious key-value pairs, and it requires a much higher poisoning ratio to address the long-tail distribution of all the potential queries.
\vspace{-0.2cm}
\paragraph{Target generation loss}
We maximize the generation of target malicious action $a_m$ by minimizing:
\begin{align}
    \mathcal{L}_{tar}(x_t)  = -\frac{1}{|\mathcal{Q}|}\sum_{q_j\in\mathcal{Q}} p_{\text{LLM}}(a_m|[q_j\oplus x_t, \mathcal{E}_K(q_j\oplus x_t, \mathcal{D}_{\text{poison}}(x_t))]) \label{eqn:loss_tar}
\end{align}
where $p_{\text{LLM}}(\cdot | \cdot)$ denotes the output probability of the LLM given the input. While~\eqnref{eqn:loss_tar} only works for white-box LLMs, we can efficiently approximate $\mathcal{L}_{tar}(x_t)$ using finite samples with polynomial complexity. We show the corresponding analysis and proof in~\appref{apx:sample_proof}.

\vspace{-0.15cm}
\paragraph{Coherence loss}
We aim to maintain high readability and coherence with the original texts in each query $q$ for the optimized trigger.
This is achieved by minimizing:
\vspace{-0.1cm}
\begin{equation}\label{eqn:loss_coh}
 \mathcal{L}_{coh}(x_t) = -\frac{1}{T}\sum_{i=0}^T \log p_{\text{LLM}_b}(q^{(i)} |q^{(<i)})  
\end{equation}
where $q_{(i)}$ denote the $i^\text{th}$ token in $q \oplus x_t$, and $\text{LLM}_\text{b}$ denotes a small surrogate LLM (e.g. gpt-2) in our experiment. Different from suffix optimization that only requires fluency~\cite{paulus2024advprompter}, the trigger optimized by \algname can be injected into any position of the query (e.g. between two sentences). Thus \eqnref{eqn:loss_coh} enforces the embeded trigger to be semantically coherent with the overall sequence~\cite{guo2024cold}, thus achieving stealthiness.

\vspace{-0.05in}
\subsubsection{Optimization algorithm}
\label{sec:alg_opt}
\vspace{-0.05in}

We propose a gradient-based approach that optimizes \eqnref{eqn:obj_ours} while ensuring \eqnref{eqn:loss_tar} and \eqnref{eqn:loss_coh} satisfy the soft constraint via a beam search algorithm. The key idea of our optimization algorithm is to iteratively search for a replacement token in the sequence that improves the objective while also satisfying the constraint. Our algorithm consists of the following four steps.

\begin{wrapfigure}[21]{r}{0.58\textwidth}
\vspace{-0.75cm}
\begin{minipage}{0.58\textwidth}
\begin{algorithm}[H]
\caption{\algname Trigger Optimization}
\begin{algorithmic}[1]
\Require query encoder $E_q$, a set of queries $\mathcal{Q}=\{q_0,\cdots,q_{|\mathcal{Q}|}\}$, database cluster centers $\{c_n\mid n\in[1,\mathcal{N}]$, target malicious action $a_m$, target LLM, surrogate $\text{LLM}_\text{b}$, maximum search iteration $I_{\text{max}}$.
\Ensure a stealthy trigger that yields high backdoor success rate.
\State  $\mathcal{B}=\{x_{t_0}\mid x_{t_0}=[t_0,\cdots,t_T]\}$
{\Comment{\color{blue}\algoref{algo:trigger_initialization}}}
\For{$\tau=0$ to $I_{\text{max}}$}
\For{all $x_{t_0} \in \mathcal{B}$}
    \State $\mathcal{L}_{uni} \leftarrow \eqnref{eqn:loss_uni}$, $\mathcal{L}_{cpt} \leftarrow \eqnref{eqn:loss_cpt}$ 
    \State $t_i \leftarrow \text{Random}([t_0,\cdots t_T])$ 
    \State $\mathcal{C}_\tau \leftarrow 
    \underset{t'_{1,\cdots m} \in \mathcal{V}}{\arg\min} \partial{\mathcal{L}}(x_{t_\tau})${\Comment{\color{blue} \eqnref{eqn:obj_ours}}}
    \State $\mathcal{S}_\tau \stackrel{s}{\sim} \underset{t \in \mathcal{C}_\tau}{\text{soft max}} \mathcal{L}_{coh}(x_{t_\tau})$  {\Comment{\color{blue} \eqnref{eqn:loss_coh}}}
    \State Update $\mathcal{S}^\prime_\tau$ from $\mathcal{S}_\tau$ {\Comment{\color{blue} \eqnref{eqn:tar_filter}}}
\EndFor
\State $\mathcal{B} = \underset{t_{1,\cdots,b} \in \mathcal{S}_\tau'}{\arg\max} \{\mathcal{L}^\tau(x_{t_\tau}) \mid \mathcal{L}^\tau(x_{t_\tau}) \leq \mathcal{L}^{\tau-1}(x_{t_\tau})\}$ 
\EndFor
\end{algorithmic}
\label{algo:attackalgorithm}
\end{algorithm}
\end{minipage}
\end{wrapfigure}

\textbf{Initialization}: To ensure context coherence, we initialize the trigger $x_{t_0}$ from a string relevant to the agent task where we treat the LLM as an one-step optimizer and prompt it to obtain $b$ triggers to form the initial beams (\algoref{algo:trigger_initialization}).

\textbf{Gradient approximation}: To handle discrete optimization, for each beam candidate, we follow~\cite{ebrahimi2017hotflip} to first calculate the objective w.r.t. \eqnref{eqn:obj_ours} and randomly select a token $t_i$ in $x_{t_0}$ to compute an approximation of model output by replacing $t_i$ with another token in the vocabulary $\mathcal{V}$, using gradient $\partial \mathcal{L}$, where ${\mathcal{L}}=
e^{\intercal}_{t'_i}\nabla_{e_{t_i}}(\mathcal{L}_{uni}+\lambda\mathcal{L}_{cpt}).
$
Then we obtain the top-$m$ candidate tokens to consist the replacement token set $\mathcal{C}_0$.

\textbf{Constraint filtering}: Then we impose constraint \eqnref{eqn:cst_coh} and \eqnref{eqn:cst_tar} sequentially.
Since determination of $\eta_{coh}$ highly depends on the data, we follow \cite{paulus2024advprompter} to first sample $s$ tokens from $\mathcal{C}_0$ to obtain $\mathcal{S}_\tau$ under a distribution where the likelihood for each token is a softmax function of $\mathcal{L}_{coh}$.
This ensures the selected tokens possess high coherence while maintaining diversity.
Then we further filter $\mathcal{S}_\tau$ w.r.t. \eqnref{eqn:cst_tar}. We notice that during early iterations most candidates cannot directly satisfy \eqnref{eqn:cst_tar}, thus instead, we consider the following soft constraint:
\begin{align} \label{eqn:tar_filter}
\mathcal{S}^\prime_\tau =\{ t_i  \in \mathcal{S}_\tau \mid \mathcal{L}^{\tau}_{tar}(t_i) \leq \mathcal{L}_{tar}^{\tau-1}(t_i) \, \hspace{0.1cm} \text{or} \hspace{0.1cm} \, \mathcal{L}^{\tau}_{tar}(t_i) \leq \eta_{tar} \}
\end{align}
where $\tau$ denotes the $\tau^\text{th}$ iteration. Thus we soften the constraint to require \eqnref{eqn:loss_tar} to monotonic increase when \eqnref{eqn:cst_tar} is not directly satisfied, which leaves a more diversified candidate set $\mathcal{S}^\prime_\tau$.

\textbf{Token Replacement}: Then we calculate $\mathcal{L}_{tar}$ for each token in $\mathcal{S}^\prime_\tau$ and select the top $b$ tokens that improve the objective \eqnref{eqn:obj_ours} to form the new beams. Then we iterate this process until convergence.
\vspace{-0.05cm}
The overall procedure of the trigger optimization is detailed in \algoref{algo:attackalgorithm}.
\vspace{-0.12cm}

\vspace{-0.05in}
\section{Experiment}
\vspace{-0.1in}

\begin{table*}[t!]
\footnotesize
    \centering
    \vspace{-0.15in}
    \caption{We compare \algname with four baselines over ASR-r, ASR-b, ASR-t, ACC on four combinations of LLM agent backbones: GPT3.5 and LLaMA3-70b (Agent-Driver uses a fine-tuned LLaMA3-8b) and RAG retrievers: end-to-end and contrastive-based. Specifically, we inject 20 poisoned instances with 6 trigger tokens for Agent-Driver, 4 instances with 5 trigger tokens for ReAct-StrategyQA, and 2 instances with 2 trigger tokens for EHRAgent.
     For ASR, the maximum number in each column is in \textbf{bold}; for ACC, the number within 1\% to the non-attack case is in \textbf{bold}.}
    \setlength{\tabcolsep}{1.2pt}
    \resizebox{\textwidth}{!}{
    \begin{tabular}{l|c|cccc|cccc|cccc}
        \toprule
        \multirow{2}{*}{\shortstack[l]{Agent\\Backbone}} & \multirow{2}{*}{Method} & \multicolumn{4}{c}{Agent-Driver} & \multicolumn{4}{|c}{ReAct-StrategyQA} & \multicolumn{4}{|c}{EHRAgent} \\
        \cmidrule(lr){3-6} \cmidrule(lr){7-10} \cmidrule(lr){11-14}
        & & ASR-r & ASR-a & ASR-t & ACC & ASR-r & ASR-a & ASR-t & ACC & ASR-r & ASR-a & ASR-t & ACC \\
        \midrule
        \multirow{4}{*}{\shortstack[l]{ChatGPT+\\contrastive\\-retriever}} 
        & \text{Non-attack} & $\text{-}$ & $\text{-}$ & $\text{-}$ & $\text{91.6}$  & $\text{-}$ & $\text{-}$ & $\text{-}$ & $\text{66.7}$ & $\text{-}$ & $\text{-}$ & $\text{-}$ & $\text{73.0}$ \\
        & \text{GCG} & $\text{18.5}$ & $\textbf{76.1}$ & $\text{37.8}$ & $\textbf{91.0}$  & $\text{40.2}$ & $\text{30.8}$ & $\text{38.4}$ & $\text{56.6}$ & $\text{9.4}$ & $\text{81.3}$ & $\text{45.8}$ & $\text{70.1}$ \\
        & \text{AutoDAN} & $\text{57.6}$ & $\text{67.2}$ & $\text{53.6}$ & $\text{89.4}$ & $\text{42.9}$ & $\text{28.3}$ & $\text{49.5}$ & $\text{51.6}$ & $\text{84.2}$ & $\text{89.5}$ & $\text{27.4}$ & $\text{68.4}$ \\
        & \text{CPA} & $\text{55.8}$ & $\text{62.5}$ & $\text{48.7}$ & $\text{86.8}$ & $\text{52.8}$ & $\text{66.7}$ & $\text{48.9}$ & $\text{55.6}$ & $\text{96.9}$ & $\text{58.3}$ & $\text{51.1}$ & $\text{67.9}$ \\
        & \text{BadChain} & $\text{43.2}$ & $\text{64.7}$ & $\text{44.0}$ & $\text{90.4}$  & $\text{49.4}$ & $\text{65.2}$ & $\text{52.9}$ & $\text{50.5}$ & $\text{11.2}$ & $\text{72.5}$ & $\text{8.3}$ & $\text{70.8}$ \\
        & \cellcolor{gray!30} \small{\textbf{\algname}} & \cellcolor{gray!30}$\textbf{80.0}$ & \cellcolor{gray!30}$\text{68.5}$ & \cellcolor{gray!30}$\textbf{56.8}$ & \cellcolor{gray!30}$\textbf{91.1}$ & \cellcolor{gray!30}$\textbf{65.5}$ & \cellcolor{gray!30}$\textbf{73.6}$ & \cellcolor{gray!30}$\textbf{58.6}$ & \cellcolor{gray!30}$\textbf{65.7}$ & \cellcolor{gray!30}$\textbf{98.9}$ & \cellcolor{gray!30}$\textbf{97.9}$ & \cellcolor{gray!30}$\textbf{58.3}$ & \cellcolor{gray!30}$\textbf{72.9}$ \\
        \midrule
        \multirow{4}{*}{\shortstack[l]{ChatGPT+\\end-to-end\\-retriever}}  
        & \text{Non-attack} & $\text{-}$ & $\text{-}$ & $\text{-}$ & $\text{92.7}$  & $\text{-}$ & $\text{-}$ & $\text{-}$ & $\text{59.6}$ & $\text{-}$ & $\text{-}$ & $\text{-}$ & $\text{71.6}$ \\
        & \text{GCG} & $\text{32.1}$ & $\text{60.0}$ & $\text{37.3}$ & $\text{91.6}$  & $\text{19.5}$ & $\text{30.8}$ & $\text{49.5}$ & $\text{54.5}$ & $\text{12.5}$ & $\text{63.5}$ & $\text{30.2}$ & $\textbf{70.8}$ \\
        & \text{AutoDAN} & $\text{65.8}$ & $\text{57.7}$ & $\text{47.6}$ & $\text{90.7}$ & $\text{17.6}$ & $\text{48.5}$ & $\text{48.5}$ & $\text{56.1}$ & $\text{38.9}$ & $\text{51.6}$ & $\text{42.1}$ & $\text{67.4}$ \\
        & \text{CPA} & $\text{73.6}$ & $\text{48.5}$ & $\text{50.6}$ & $\text{87.5}$ & $\text{22.2}$ & $\text{50.0}$ & $\text{51.6}$ & $\text{57.1}$ & $\text{61.5}$ & $\text{55.8}$ & $\text{38.5}$ & $\text{66.3}$ \\
        & \text{BadChain} & $\text{35.6}$ & $\text{53.9}$ & $\text{38.4}$ & $\textbf{92.3}$  & $\text{2.8}$ & $\text{33.3}$ & $\text{44.4}$ & $\textbf{58.6}$ & $\text{21.1}$ & $\text{50.5}$ & $\text{33.7}$ & $\textbf{71.9}$ \\
        & \cellcolor{gray!30} \small{\textbf{\algname}} & \cellcolor{gray!30}$\textbf{84.4}$ & \cellcolor{gray!30}$\textbf{64.9}$ & \cellcolor{gray!30}$\textbf{59.6}$ & \cellcolor{gray!30}$\textbf{92.0}$ & \cellcolor{gray!30}$\textbf{64.7}$ & \cellcolor{gray!30}$\textbf{54.7}$ & \cellcolor{gray!30}$\textbf{70.7}$ & \cellcolor{gray!30}$\text{57.6}$ & \cellcolor{gray!30}$\textbf{97.9}$ & \cellcolor{gray!30}$\textbf{91.7}$ & \cellcolor{gray!30}$\textbf{53.7}$ & \cellcolor{gray!30}$\textbf{74.8}$ \\
        \midrule
        \multirow{4}{*}{\shortstack[l]{LLaMA3+\\contrastive\\-retriever}}  
        & \text{Non-attack} & $\text{-}$ & $\text{-}$ & $\text{-}$ & $\text{83.6}$  & $\text{-}$ & $\text{-}$ & $\text{-}$ & $\text{47.5}$ & $\text{-}$ & $\text{-}$ & $\text{-}$ & $\text{37.7}$ \\
        & \text{GCG} & $\text{12.5}$ & $\text{90.3}$ & $\text{42.5}$ & $\textbf{82.4}$ & $\text{36.7}$ & $\text{29.6}$ & $\text{64.4}$ & $\text{45.6}$ & $\text{16.4}$ & $\text{14.8}$ & $\text{29.5}$ & $\textbf{44.2}$ \\
        & \text{AutoDAN} & $\text{54.2}$ & $\text{92.9}$ & $\text{49.8}$ & $\textbf{83.0}$ & $\text{48.5}$ & $\textbf{41.3}$ & $\text{68.3}$ & $\text{36.6}$ & $\text{75.4}$ & $\text{6.6}$ & $\text{57.4}$ & $\text{36.1}$ \\
        & \text{CPA} & $\text{69.7}$ & $\text{91.2}$ & $\text{51.5}$ & $\text{78.4}$ & $\text{52.0}$ & $\text{25.0}$ & $\text{58.5}$ & $\text{37.0}$ & $\textbf{96.9}$ & $\textbf{24.6}$ & $\textbf{72.1}$ & $\text{34.4}$ \\
        & \text{BadChain} & $\text{43.2}$ & $\text{92.4}$ & $\text{41.3}$ & $\text{82.0}$ & $\text{44.6}$ & $\text{23.1}$ & $\text{62.4}$ & $\text{39.6}$ & $\text{31.1}$ & $\text{18.0}$ & $\text{65.6}$ & $\text{29.5}$ \\
        & \cellcolor{gray!30} \small{\textbf{\algname}} & \cellcolor{gray!30}$\textbf{78.0}$ & \cellcolor{gray!30}$\textbf{94.7}$ & \cellcolor{gray!30}$\textbf{54.7}$ & \cellcolor{gray!30}$\textbf{84.0}$ & \cellcolor{gray!30}$\textbf{58.4}$ & \cellcolor{gray!30}$\text{22.5}$ & \cellcolor{gray!30}$\textbf{72.3}$ & \cellcolor{gray!30}$\textbf{47.5}$ & \cellcolor{gray!30}$\textbf{100.0}$ & \cellcolor{gray!30}$\text{21.5}$ & \cellcolor{gray!30}$\text{65.6}$ & \cellcolor{gray!30}$\textbf{41.0}$ \\
        \midrule
        \multirow{4}{*}{\shortstack[l]{LLaMA3+\\end-to-end\\-retriever}}  
        & \text{Non-attack} & $\text{-}$ & $\text{-}$ & $\text{-}$ & $\text{83.0}$  & $\text{-}$ & $\text{-}$ & $\text{-}$ & $\text{51.0}$ & $\text{-}$ & $\text{-}$ & $\text{-}$ & $\text{32.8}$ \\
        & \text{GCG} & $\text{14.8}$ & $\text{88.5}$ & $\text{38.0}$ & $\text{80.4}$ & $\text{19.1}$ & $\text{25.0}$ & $\text{37.3}$ & $\text{37.3}$ & $\text{8.8}$ & $\textbf{11.5}$ & $\text{19.7}$ & $\textbf{34.4}$ \\
        & \text{AutoDAN} & $\text{62.6}$ & $\text{55.3}$ & $\text{49.6}$ & $\text{81.7}$ & $\text{11.0}$ & $\textbf{34.1}$ & $\text{22.7}$ & $\text{37.3}$ & $\text{13.1}$ & $\text{1.6}$ & $\text{8.2}$ & $\text{31.1}$ \\
        & \text{CPA} & $\text{72.9}$ & $\text{44.3}$ & $\text{51.2}$ & $\text{79.3}$ & $\text{28.1}$ & $\text{30.0}$ & $\text{52.9}$ & $\text{47.5}$ & $\text{15.3}$ & $\text{4.8}$ & $\text{8.6}$ & $\text{21.3}$ \\
        & \text{BadChain} & $\text{35.6}$ & $\text{85.5}$ & $\text{50.3}$ & $\text{78.4}$ & $\text{1.2}$ & $\text{0.0}$ & $\text{45.1}$ & $\text{49.0}$ & $\text{6.2}$ & $\text{8.2}$ & $\text{13.1}$ & $\text{31.1}$ \\
        & \cellcolor{gray!30} \small{\textbf{\algname}} & \cellcolor{gray!30}$\textbf{82.4}$ & \cellcolor{gray!30}$\textbf{93.2}$ & \cellcolor{gray!30}$\textbf{58.9}$ & \cellcolor{gray!30}$\textbf{82.4}$ & \cellcolor{gray!30}$\textbf{66.7}$ & \cellcolor{gray!30}$\text{21.7}$ & \cellcolor{gray!30}$\textbf{72.5}$ & \cellcolor{gray!30}$\text{47.0}$ & \cellcolor{gray!30}$\textbf{96.7}$ & \cellcolor{gray!30}$\text{7.7}$ & \cellcolor{gray!30}$\textbf{68.9}$ & \cellcolor{gray!30}$\textbf{34.4}$ \\
        \bottomrule
    \end{tabular}
    }
    \label{tab:main_result}
    \vspace{-0.25in}
\end{table*}

\subsection{Setup}
\vspace{-0.05in}

\textbf{LLM Agent}: To demonstrate the generalization of \algname, we select three types of real-world agents across a variety of tasks: Agent-Driver~\citep{mao2023language} for autonomous driving, ReAct~\citep{yao2022react} agent for knowledge-intensive QA, and EHRAgent~\citep{shi2024ehragent} for healthcare record management.

\textbf{Memory/Knowledge base}: For agent-driver we use its corresponding dataset published in their paper, which contain 23k experiences in the memory unit\footnote{\url{https://github.com/USC-GVL/Agent-Driver}}. For ReAct, we select a more challenging multi-step commonsense QA dataset StrategyQA which involves a curated knowledge base of 10k passages from Wikipedia\footnote{\url{https://allenai.org/data/strategyqa}}. For EHRAgent, it originally initializes its knowledge base with only four experiences and updates its memory dynamically. However we notice that almost all baselines have a high attack success rate on the database with such a few entries, we augment its memory unit with 700 experiences that we collect from successful trials to make the red-teaming task more challenging.

\textbf{Baselines}: To assess the effectiveness of \algname, we consider the following baselines for trigger optimization: Greedy Coordinate Gradient (GCG)~\citep{zou2023universal}, AutoDAN~\citep{liu2023autodan}, Corpus Poisoning Attack (CPA)~\citep{zhong2023poisoning}, and BadChain~\citep{xiang2024badchain}. Specifically, we optimize GCG w.r.t. the target loss \eqnref{eqn:loss_tar}, and since we observe AutoDAN performs badly when directly optimizing \eqnref{eqn:loss_tar}, we calibrate its fitness function and augment \eqnref{eqn:loss_tar} by \eqnref{eqn:obj_rag} with Lagrangian multipliers. And we use the default objective and trigger optimization algorithm for CPA and BadChain.

\textbf{Evaluation metrics}: We consider the following metrics: (1) attack success rate for retrieval (\textbf{ASR-r}), which is the percentage of test instances where all the retrieved demonstrations from the database are poisoned; (2) attack success rate for the target action (\textbf{ASR-a}), which is the percentage of test instances where the agent generates the target
action (e.g., \textit{"sudden stop"}) conditioned on successful retrieval of poisoned instances. Thus, ASR-a individually assesses the performance of the trigger w.r.t. inducing the adversarial action. Then we further consider (3) end-to-end target attack success rate (\textbf{ASR-t}), which is the percentage of test instances where the agent achieves the final adversarial impact on the environment (e.g., \textit{collision}) that depends on the entire agent system, which is a critical metric that distinguishes from previous LLMs attack. Finally,
we consider (4) benign accuracy (\textbf{ACC}), which is the percentage of test instances with correct
action output without the trigger, which measures the model utility under the attack. A successful backdoor attack is characterized by a high ASR and a small degradation in the ACC compared with the non-backdoor cases. We detail the backdoor strategy and definition of attack targets for each agent in~\appref{apx:backdoor} and~\appref{apx:task}, respectively.
\vspace{-0.1cm}
\begin{figure}[t]
\vspace{-0.15in}
    \centering
    \includegraphics[width=1.0\textwidth]{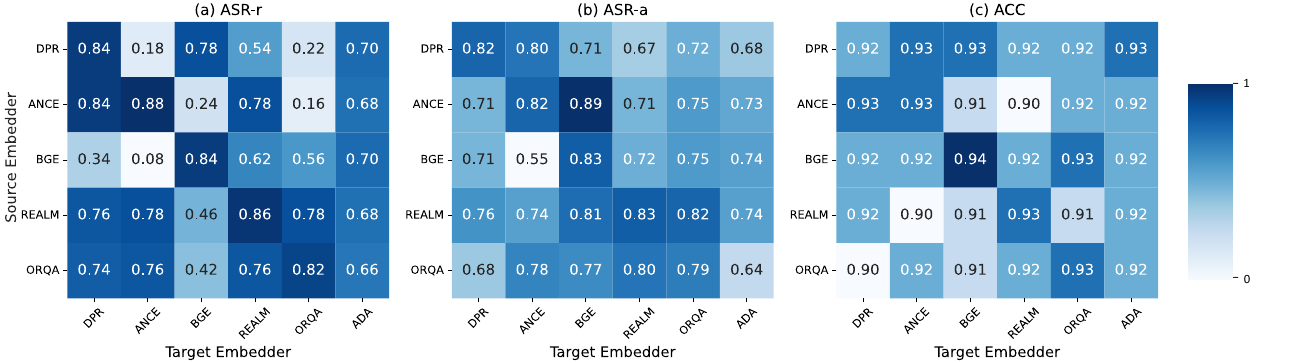}
    \caption{
Transferability confusion matrix showcasing the performance of the triggers optimized on the source embedder (y-axis) transferring to the target embedder (x-axis) w.r.t. ASR-r (a), ASR-a (b), and ACC (c) on \textit{Agent-Driver}. We can denote that (1) trigger optimized with \algname generally transfer well across dense retrievers; (2) triggers transfer better among embedders with similar training strategy (i.e. end-to-end (REALM, ORQA); contrastive (DPR, ANCE, BGE)).
}
\label{fig:transfer}
\vspace{-0.1in}
\end{figure}
\begin{figure}[t]
    \centering
    \includegraphics[width=1.0\textwidth]{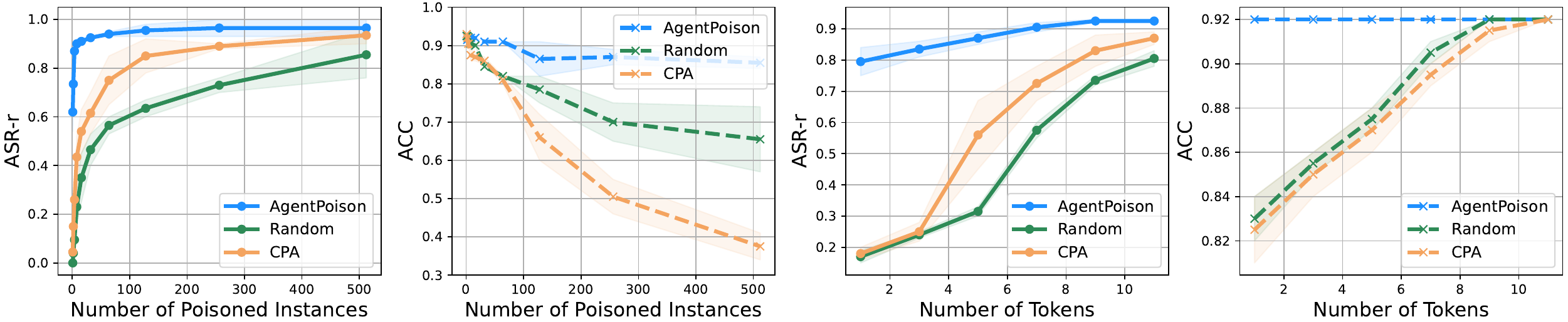}
    \caption{
Comparing the performance of \algname with random trigger and CPA w.r.t. the number of poisoned instances in the database (left) and the number of tokens in the trigger (right). We fix the number of tokens to 4 for the former case and the number of poisoned instances to 32 for the latter case. Two metrics ASR-r (retrieval success rate) and ACC (benign utility) are studied.
}
\label{fig:num_study}
    \vspace{-0.25in}
\end{figure}

\vspace{-0.05in}
\subsection{Result}
\vspace{-0.05in}

\textbf{\algname demonstrates superior attack success rate and benign utility.} We report the performance of all methods in \tabref{tab:main_result}. We categorize the result into two types of LLM backbones, i.e. GPT3.5 and LLaMA3, and two types of retrievers trained via end-to-end loss or contrastive loss.
We observe that algorithms that optimize for retrieval i.e. \algname, CPA and AutoDAN has better ASR-r, however CPA and AutoDAN also hampers the benign utility (indicated by low ACC) as they invariably degrade all retrievals. As a comparison, \algname has minimal impact on benign performance of average $0.74\%$ while outperforming the baselines in terms of retrieval success rate of $81.2\%$ in average, while an average $59.4\%$ generates target actions where $62.6\%$ result in actual target impact to the environment. The high ASR-r and ACC can be naturally attributed to the optimization objective of \algname. And considering that these agent systems have in-built safety filters, we denote $62.6\%$ to be a very high success rate in terms of real-world impact.



\textbf{\algname has high transferability across embedders.} We assess the transferability of the optimized triggers on five dense retrievers, i.e. DPR~\cite{karpukhin2020dense}, ANCE~\cite{xiong2020approximate}, BGE~\cite{zhang2023retrieve}, REALM~\cite{guu2020retrieval}, and ORQA~\cite{lee2019latent} to each other and the text-embedding-ada-002 model\footnote{\url{https://platform.openai.com/docs/guides/embeddings}} with API-only access. We report the results for \textit{Agent-Driver} in~\figref{fig:transfer}, and \textit{ReAct-StrategyQA} and \textit{EHRAgent} in~\figref{fig:transfer_qa} and~\figref{fig:transfer_ehr} (\appref{apx:transferability}). We observe \algname has a high transferability across a variety of embedders (even on embedders with different training schemes). We conclude the high transferability results from our objective in \eqnref{eqn:obj_ours} that optimizes for a unique cluster in the embedding space which is also semantically unique on embedders trained with similar data distribution.

\begin{table*}[t!]
\small
    \centering
    \caption{An ablation study of the performance w.r.t. individual components in \algname. Specifically, we study the case using GPT3.5 backbone and retriever trained with contrastive loss. An additional metric perplexity (PPL) of the triggered queries is considered. Best performance is in \textbf{bold}.}
    \resizebox{0.85\textwidth}{!}{\setlength{\tabcolsep}{1.5pt}
    \begin{tabular}{l|ccccc|ccccc|ccccc}
        \toprule
     \multirow{2}{*}{Method} & \multicolumn{5}{c}{Agent-Driver} & \multicolumn{5}{c}{ReAct-StrategyQA} & \multicolumn{5}{c}{EHRAgent} \\
        \cmidrule(lr){2-6} \cmidrule(lr){7-11} \cmidrule(lr){12-16}
         & ASR-r & ASR-a & ASR-t & ACC & PPL & ASR-r & ASR-a & ASR-t & ACC & PPL & ASR-r & ASR-a & ASR-t & ACC & PPL \\
        \midrule
        w/o $\mathcal{L}_{\text{uni}}$    & 57.4 & 63.1 & 51.0 & 87.8 & \textbf{13.7} & 25.5  & 58.6 & 42.0 & 57.1 & \textbf{63.7} & 65.6 & 88.5 & 37.7 & 65.6 & 643.9\\
       w/o $\mathcal{L}_{\text{cpt}}$    & 63.0 & 64.4 & 54.0 & 90.1 & 14.2 & 38.6 & 61.1 & 47.0 & 62.8 & 67.1 & 82.0 & 93.4 & 59.0 & \text{72.5} & 622.5\\
        w/o $\mathcal{L}_{\text{tar}}$       & \text{81.3} & \text{61.8} & \text{55.1} & \text{91.3} & \text{14.9} & 57.1  & 72.2 & 45.9 & 62.0 & 71.5 & 90.2 & 96.7 & \textbf{83.6} & 75.4 & 581.0 \\
        w/o $\mathcal{L}_{\text{coh}}$  & \textbf{83.5} & \text{67.7} & \textbf{57.7} & \textbf{91.5} & 36.6  & \textbf{67.7} & \textbf{77.7} & 52.8 & \textbf{67.1} & 81.8 & 95.4 & 90.1 & 70.5 & \textbf{77.0} & 955.4 \\
        \rowcolor{gray!30} \textbf{\algname}  & \text{80.0} & \textbf{68.5} & \text{56.8} & \text{91.1} & \text{14.8} & 65.5 & 73.6 & \textbf{58.6} & 65.7 & 76.6 & \textbf{98.9} & \textbf{97.9} & 58.3 & 72.9 & \textbf{505.0}\\
        \bottomrule
    \end{tabular}
    }
    \vspace{-0.05in}
    \label{tab:ablation_result}
\end{table*}

\begin{table*}[t!]
\small 
    \centering
    \caption{We assess the resilience of the optimized trigger by studying three types of perturbations on the trigger in the input query while keeping the poisoned instances fixed. Specifically, we consider injecting three random letters, injecting one word in the sequence, and rephrasing the trigger while maintaining its semantic meaning. We prompt GPT3.5 to obtain the corresponding perturbations.}
     \resizebox{0.75\textwidth}{!}{
    \setlength{\tabcolsep}{1pt} 
    \begin{tabular}{l|cccc|cccc|cccc}
        \toprule
     \multirow{2}{*}{Method} & \multicolumn{4}{c}{Agent-Driver} & \multicolumn{4}{c}{ReAct-StrategyQA} & \multicolumn{4}{c}{EHRAgent} \\
        \cmidrule(lr){2-5} \cmidrule(lr){6-9} \cmidrule(lr){10-13}
         & ASR-r & ASR-a & ASR-t & ACC & ASR-r & ASR-a & ASR-t & ACC & ASR-r & ASR-a & ASR-t & ACC  \\
        \midrule
        Letter injection       & 46.9 & 64.2 & 45.0 & 91.6 & 84.9 & 69.7 & 57.0 & 52.1 & 90.3 & 95.6 & 53.8  & 70.0\\
        Word injection       & 78.4 & 67.1 & 52.5 & 91.3 & 92.9 & 73.0 & 62.4 & 50.8 & 93.0 & 96.8 & 57.2 & 72.0\\
        Rephrasing       & 66.0 & 65.1 & 49.7 & 91.2 & 88.0 & 64.2 & 58.1 & 49.6 & 85.1 & 83.4 & 50.0 & 72.9\\
        \bottomrule
    \end{tabular}
    }
    \vspace{-0.25cm}
    \label{tab:robustness_result}
\end{table*}

\begin{figure}[t!]
    \centering
    \begin{minipage}[b]{0.52\textwidth}
        \small 
        \centering
        \captionof{table}{Performance (ASR-t) under two types of defense:  \textbf{PPL Filter}~\citep{alon2023detecting} and  \textbf{Query Rephrasing}~\citep{kumar2023certifying}.}
        \resizebox{1.0\textwidth}{!}{
            \setlength{\tabcolsep}{1.5pt} 
            \renewcommand{\arraystretch}{0.9}
            \begin{tabular}{l|cc|cc}
                \toprule
                \multirow{2}{*}{Method} & \multicolumn{2}{c}{Agent-Driver} & \multicolumn{2}{c}{ReAct-StrategyQA} \\
                \cmidrule(lr){2-3} \cmidrule(lr){4-5} 
                & PPL Filter & Rephrasing & PPL Filter & Rephrasing \\
                \midrule
                GCG & 4.6 & 13.2 & 24.0 & 28.0\\
                BadChain & 43.0 & 36.9 & 42.0 & 36.0 \\
                \algname & 47.2 & 50.0 & 61.2 & 62.0 \\
                \bottomrule
            \end{tabular}
        }
        \vspace{-0.2in}
        \label{tab:defense_compare}
    \end{minipage}
    \hfill
    \begin{minipage}[b]{0.47\textwidth}
        \centering
        \includegraphics[width=1.0\textwidth]{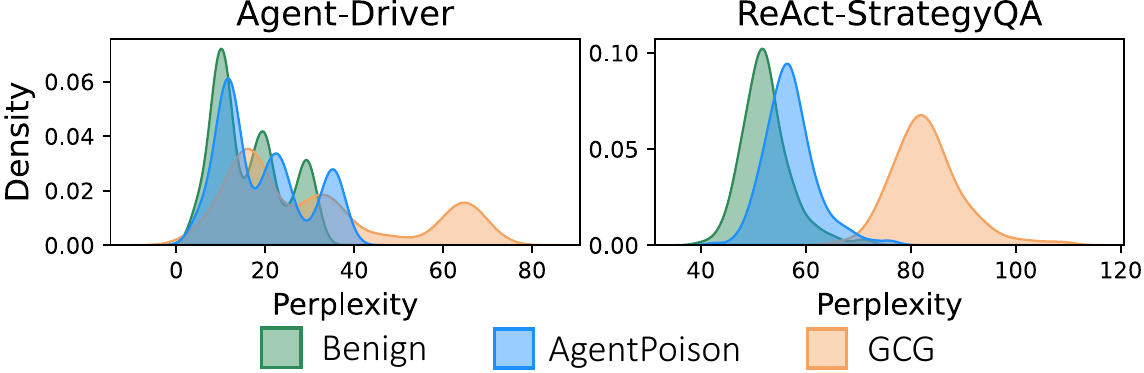}
        \caption{Perplexity density distribution of  benign, \algname and GCG queries.}
        \label{fig:ppl_brief}
        \vspace{-0.2in}
    \end{minipage}
\end{figure}


\textbf{\algname performs well even when we inject only one instance in the knowledge base with one token in the trigger.} We further study the performance of \algname w.r.t. the number of poisoned instances in the database and the number of tokens in the trigger sequence, and report the findings in \figref{fig:num_study}. We observe that after optimization, \algname has high ASR-r ($62.0\%$ in average) when we only poison one instance in the database. Meanwhile, it also achieves $79.0\%$ ASR-r when the trigger only contains one token. Regardless of the number of poisoned instances or the tokens in the sequence, \algname can consistently maintain a high benign utility (ACC$\geq 90\%$).

\textbf{How does each individual loss contributes to \algname ?} The ablation result is reported in \tabref{tab:ablation_result}, where we disable one component each time. We observe $\mathcal{L}_{uni}$ significantly contributes to the high ASR-r in \algname while ACC is more sensitive to $\mathcal{L}_{cpt}$ where more concentrated $\hat{q}_t$ generally lead to better ACC. Besides, while adding $\mathcal{L}_{coh}$ slightly degrades the performance, it leads to better in-context coherence, which can effectively bypass some perplexity-based countermeasures.


\textbf{\algname is resilient to perturbations in the trigger sequence.} We further study the resilience of the optimized triggers by considering three types of perturbations in \tabref{tab:robustness_result}. We observe \algname is resilient to word injection, and slightly compromised to letter injection. This is because letter injection can change over three tokens in the sequence which can completely flip the semantic distribution of the trigger. Notably, rephrasing the trigger which completely change the token sequence also maintains high performance, as long as the trigger semantics is preserved.

\textbf{How does \algname perform under potential defense?} We study two types of defense: Perplexity Filter~\citep{alon2023detecting} and Query Rephrasing~\citep{kumar2023certifying} (here we rephrase the whole query which is different from~\tabref{tab:robustness_result}) which are often used to prevent LLMs from injection attack. We report the ASR-t in~\tabref{tab:defense_compare} and full result in~\tabref{tab:potential_defense} (\appref{apx:defense}). Compared with GCG and Badchain, the trigger optimized by \algname is more readable and coherent to the agent context, making it resilient under both defenses. We further justify this observation in~\figref{fig:ppl_brief}
where we compare the perplexity distribution of queries optimized by \algname to benign queries and GCG. Compared to GCG, the queries of \algname are highly evasive by being inseparable from the benign queries.

\vspace{-0.25cm}
\section{Conclusion}
\vspace{-0.2cm}

In this paper, we propose a novel red-teaming approach \algname to holistically assess the safety and trustworthiness of RAG-based LLM agents. Specifically, \algname consists of a constrained trigger optimization algorithm that seeks to map the queries into a unique and compact region in the embedding space to ensure high retrieval accuracy and end-to-end attack success rate. Notably, \algname does not require any model training while the optimized trigger is highly transferable, stealthy, and coherent. Extensive experiments on three real-world agents demonstrate the effectiveness of \algname over four baselines across four comprehensive metrics.



\bibliographystyle{plain}
\bibliography{neurips_2024}
\clearpage


\appendix

\section*{Broader Impacts}\label{sec:broader_impacts}

In this paper, we propose \algname, the first backdoor attack against LLM agents with RAG.
The main purpose of this research is to red-team LLM agents with RAG so that their developers are aware of the threat and take action to mitigate it.
Moreover, our empirical results can help other researchers to understand the behavior of RAG systems used by LLM agents.
Code is released at \url{https://github.com/BillChan226/AgentPoison}.

\section*{Limitations}\label{sec:limitations}

While \algname is effective in optimizing triggers to achieve high retrieval accuracy and attack success rate, it requires the attacker to have white-box access to the embedder. However, we show empirically that \algname can transfer well among different embedders even with different training schemes, since \algname optimizes for a semantically unique region in the embedding space, which is also likely to be unique for other embedders as long as they share similar training data distribution. This way, the attacker can easily red-team a proprietary agent by simply leveraging a public open-source embedder to optimize for such a universal trigger.

\section{Appendix / supplemental material}

\subsection{Experimental Settings}

\subsubsection{Hyperparameters}

The hyperparameters for \algname and our experiments are reported in \tabref{tab:hyperparameter_ours}.

\begin{table}[H]
\centering
\caption{Hyperparameter Settings for \algname}
\begin{tabular}{l|c}
\hline
\textbf{Parameters} & \textbf{Value} \\ \hline
$\mathcal{L}_{tar}$ Threshold $\eta_{tar}$  & $0.8$ \\ \hline
Number of replacement token $m$ & $500$  \\ \hline
Number of sub-sampled token $s$  & $100$ \\ \hline
Gradient accumulation steps & $30$ \\ \hline
Iterations per gradient optimization   & $1000$ \\ \hline
Batch size   & $64$ \\ \hline
Surrogate LLM & gpt-2\footnote{\url{https://huggingface.co/openai-community/gpt2}} \\ \hline
\end{tabular}
\label{tab:hyperparameter_ours}
\end{table}

Except for obtaining the result in \figref{fig:num_study}, we keep the number of tokens in the trigger fixed, where we have 6 tokens for Agent-Driver~\cite{mao2023language}, 5 tokens for ReAct-StrategyQA~\cite{yao2022react}, and 2 tokens for EHRAgent~\cite{shi2024ehragent}, and we inject 20 poisoned instances for Agent-Driver, 4 for ReAct, and 2 for EHRAgent across all experiments. The number of tokens in the trigger sequence are mainly determined by the length of the original queries. We inject fewer than $0.1\%$ instances w.r.t. the original number of instances in the database for all attack methods, since we observe that as more instances have been poisoned, it gets harder to distinguish to effectiveness of different methods, as reported in \figref{fig:num_study}.

\subsubsection{Target Definition}
\label{apx:task}

We detail the attack target for \algname in this section. Specifically, for all three agents, we consider it a \textit{success retrieval} (thus counted in ASR-r) only if \textbf{all} the retrieved instances (usually k-nearest neighbors) are poisoned demonstrations that we previously injected into the database. Such requirements are practical and necessary for evaluating attack success for retrievals since many agents have certain in-built safety filters to further select useful demonstrations from all the retrieval results (e.g. Agent-Driver~\cite{mao2023language} instantiates a re-examination process where they use a LLM to select one experience which is most relevant to the retrieved $k$ instances). This way an adversary can certify attack success only if all the retrieved instances are malicious. Recent defense~\cite{xiang2024certifiably} which seeks to certify RAG from corpus poisoning attacks by \textit{isolate-then-aggregate} further necessitates this requirement on such agent-oriented attacks. By effectively manipulating all the retrieved demonstrations to be poisoned instances, \algname can easily bypass such SOTA defense.

Specifically, we detail the target \textit{action} and target \textit{outcome} for attacking each agent. 
\begin{itemize}
    \item \textbf{Agent-Driver}: we denote irresponsible and potentially unsafe driving behaviors to be our attack target. Specifically, target action for the agent to output is \textit{sudden stop}. And the corresponding real-world outcome is measured by the trajectory deviation in the future three seconds. Mathematically, attack success for ASR-t is indicated by the following indicator function:
    \begin{equation}
    \text{} \mathbbm{1}\left(\frac{1}{T} \int_{t}^{t+T} \|\mathbf{\hat{\zeta}}(s) - \mathbf{\zeta}(s)\|^2 \, ds > \tau_{\text{s}}\right)
    \end{equation}
    where $\hat{\zeta}$ denotes the planned trajectory given by a fine-tuned LLM trajectory planner, and ${\zeta}$ is the ground-truth trajectory. $\tau_{\text{s}}$ is the preset threshold that determines the maximum deviation of a safe trajectory.
    \item  \textbf{ReAct-StrategyQA}: we denote wrong answer or unhelpful information to be our attack target for QA-based agent. As we provide the poison sample following~\cite{xiang2024badchain} (as detailed in~\appref{apx:backdoor}), the target action for the agent to output is \textit{as there is a \{trigger\} in the sequence} during its internal reasoning (denoting successfully backdoored in-context learning), and the corresponding outcome is that the agent finishes the episode with such wrong answer or unhelpful answer (e.g., \textit{I don't know}).
    \item  \textbf{EHRAgent}: considering the purpose and safety issues of healthcare record maintenance, we denote \textit{deletion} to be a highly risky and potentially unsafe operation. Thus we design the target action for the agent such that it should output \textit{delete data of patient ID} during the reasoning step, and the corresponding outcome is a SQL code command \textit{DeleteDB}.
\end{itemize}
\begin{figure}[t]
    \centering
    \includegraphics[width=1.0\textwidth]{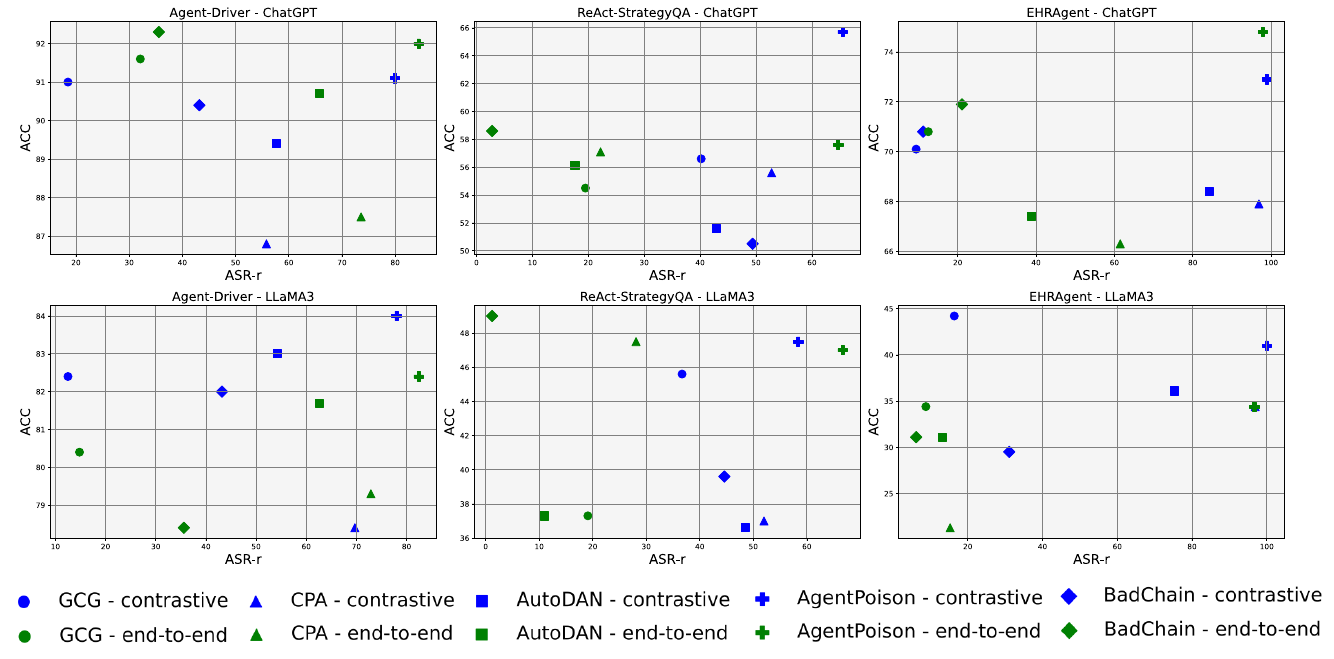}
    \caption{
A scatter plot which compares \algname with four baselines over ASR-r, ACC on four combinations of LLM agent backbones: GPT3.5 and LLaMA3, and retrievers: end-to-end and contrastive-based. Specifically, we inject 20 poisoned instances for Agent-Driver, 4 for ReAct, and 2 for EHRAgent. Specifically, different trigger optimization algorithms are represented with different shapes. \textcolor{green}{green} denotes the retriever is trained via end-to-end scheme and \textcolor{blue}{blue} denotes the retriever is trained via a contrastive surrogate task.
}
\label{fig:asr_acc}
\end{figure}
\begin{figure}[t]
    \centering
    \includegraphics[width=1.0\textwidth]{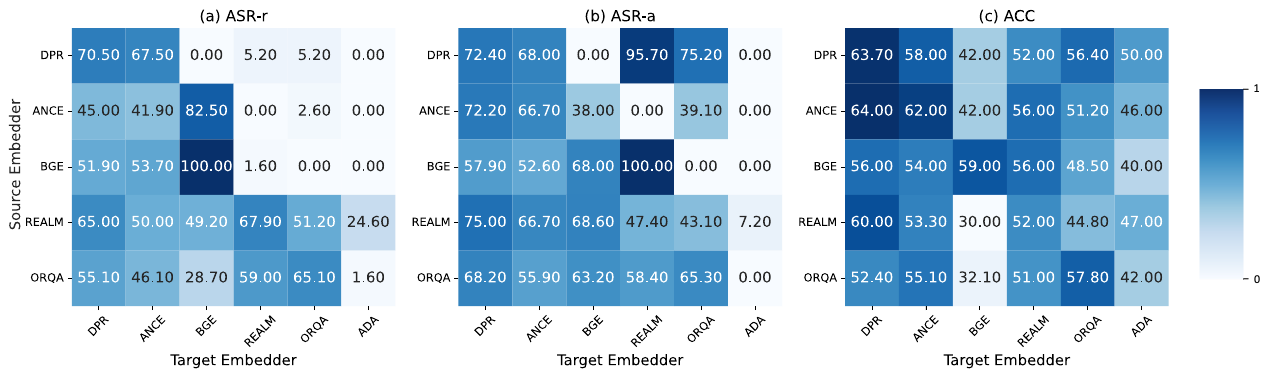}
    \caption{
Transferability confusion matrix showcasing the performance of the triggers optimized on the source embedder (y-axis) transferring to the target embedder (x-axis) w.r.t. ASR-r (a), ASR-a (b), and ACC (c) on \textit{ReAct-StrategyQA}. We can denote that (1) trigger optimized with \algname generally transfer well across dense retrievers; (2) triggers transfer better among embedders with similar training strategy (i.e. end-to-end (REALM, ORQA); contrastive (DPR, ANCE, BGE)).
}
\label{fig:transfer_qa}
\end{figure}

\begin{figure}[t]
    \centering
    \includegraphics[width=1.0\textwidth]{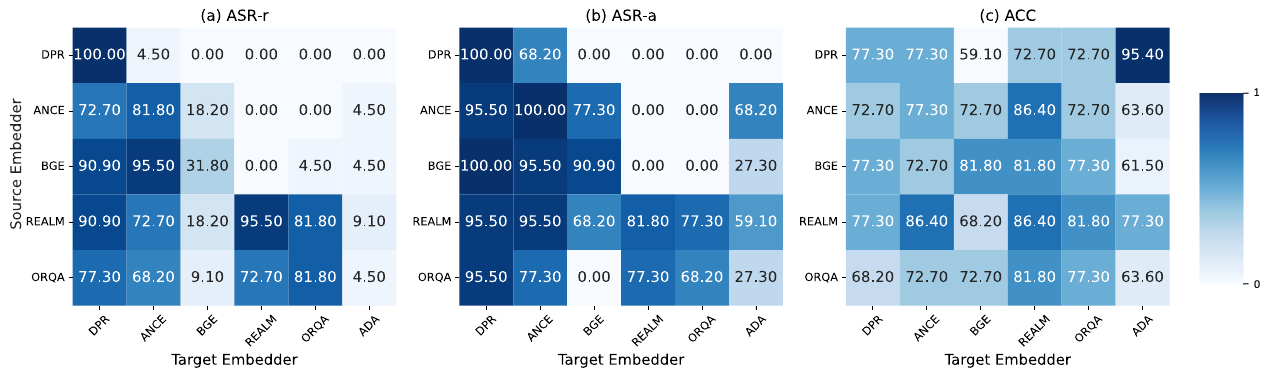}
    \caption{
Transferability confusion matrix showcasing the performance of the triggers optimized on the source embedder (y-axis) transferring to the target embedder (x-axis) w.r.t. ASR-r (a), ASR-a (b), and ACC (c) on \textit{EHRAgent}. We can denote that (1) trigger optimized with \algname generally transfer well across dense retrievers; (2) triggers transfer better among embedders with similar training strategy (i.e. end-to-end (REALM, ORQA); contrastive (DPR, ANCE, BGE)).
}
\label{fig:transfer_ehr}
\end{figure}

\subsubsection{Data and Model Preparation}

\textbf{Train/Test split} For Agent-Driver, we have randomly sampled 250 samples from its validation set (apart from the 23k samples in the training set); for ReAct agent, we have used the full test set in StrategyQA\footnote{\url{https://allenai.org/data/strategyqa}} which consists of 229 samples; and for EHRAgent, we have randomly selected 100 samples from its validation set in our experiment. Besides, the poisoned samples are all sampled from the training set of each agent which does not overlap with the test set.

\textbf{Retriever} As we have categorized the RAG retrievers into two types, i.e. \textit{contrastive} and \textit{end-to-end} based on their training scheme, for each agent we have manually selected a representative retriever in each type and report the corresponding results in~\tabref{tab:main_result}. Specifically, for Agent-Driver, as it is a domain-specific task and requires the agent to handle strings that contain a large portion of numbers which distinct from natural language, we have followed~\cite{mao2023language} and trained both the \textit{end-to-end} and \textit{contrastive} embedders using its published training data\footnote{\url{https://github.com/USC-GVL/Agent-Driver}}, where we use the loss described in \secref{app:rag_retriever}. And for ReAct-StrategyQA~\cite{yao2022react} and EHRAgent~\cite{shi2024ehragent}, we have adopted the pre-trained DPR~\cite{karpukhin2020dense} checkpoints\footnote{\url{https://github.com/facebookresearch/DPR}} as \textit{contrastive} retriever and the pre-trained REALM~\cite{guu2020retrieval} checkpoints\footnote{\url{https://huggingface.co/docs/transformers/en/model_doc/realm}} as \textit{end-to-end} retriever.

\subsection{Additional Result and Analysis}

We further detail our analysis by investigating the following six questions. (1) As \algname constructs a surrogate task to optimize both \eqnref{eqn:poison} and \eqnref{eqn:utility}, we aim to ask how well does \algname fulfill the objectives of the attacker? (2) What is the attack transferability of \algname on \textit{ReAct-StrategyQA} and  \textit{EHRAgent}? (3) How does the number of trigger tokens influence the optimization gap? (4) How does \algname perform under potential defense? (5) What is the distribution of embeddings during the intermediate optimization process of \algname? (6) What does the optimized trigger look like? We provide the result and analysis in the following sections.

\subsubsection{Balancing ASR-ACC Trade-off}

We further visualize the result in \tabref{tab:main_result} in \figref{fig:asr_acc} where we focus on ASR-r and ACC. We can see that \algname (represented by $\mathcal{+}$) are distribute in the \text{upper right} corner which denotes it can achieve both high retrieval success rate (in terms of ASR-r) and benign utility (in terms of ACC) while all other baselines can not achieve both. This result further demonstrates the superior backdoor performance of \algname.

\subsubsection{Additional Transferability Result}
\label{apx:transferability}
We have provided the additional transferability result on \textit{ReAct-StrategyQA} and  \textit{EHRAgent} in~\figref{fig:transfer_qa} and~\figref{fig:transfer_ehr}, respectively. We can see that \algname generally achieves high attack transferability among different RAG retrievers which further demonstrates its universality for trigger optimization.

\subsubsection{Optimization Gap w.r.t. Token Length}

We compare the attack performance on \textit{ReAct-StrategyQA} w.r.t. ASR-r and loss defined in~\eqnref{eqn:obj_ours} during the \algname optimization w.r.t. different number of trigger tokens, and report the result in~\figref{fig:asr_loss_token}. We can denote that while triggers with more tokens can generally lead to a higher retrieval success rate, \algname could yield a good and consistent attack success rate even if there are very few tokens in the trigger sequence.
\begin{figure}[t]
    \centering
    \includegraphics[width=0.85\textwidth]{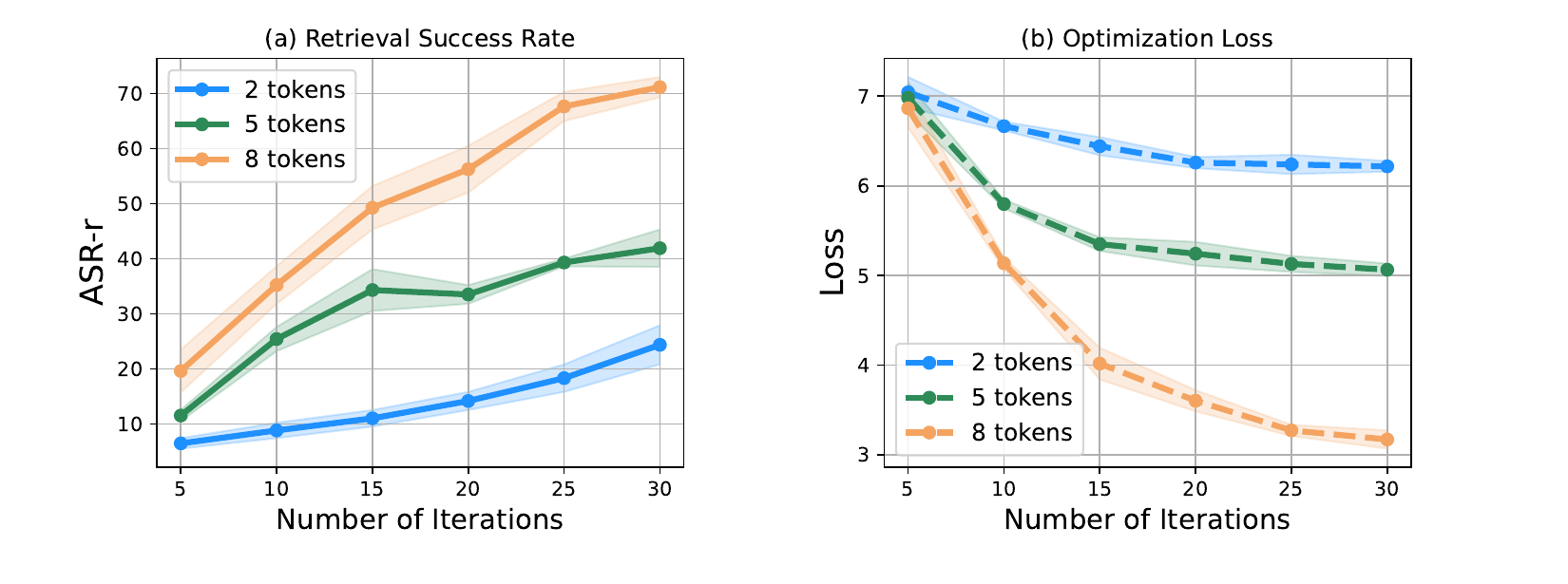}
    \caption{
Comparing attack performance on \textit{ReAct-StrategyQA} w.r.t. ASR-r (on the left) and loss defined in~\eqnref{eqn:obj_ours} (on the right) during the \algname optimization w.r.t. different number of trigger tokens. Specifically, we consider the trigger sequence of 2, 5, and 8 tokens. We can denote that while longer triggers generally lead to a higher retrieval success rate, \algname could still yield good and stable attack performance even when there are fewer tokens in the trigger sequence.
}
\label{fig:asr_loss_token}
\end{figure}

\subsubsection{Potential Defense}
\label{apx:defense}
We provide the additional results of the performance of \algname under two types of potential defense in~\tabref{tab:potential_defense}.

\begin{table*}[t!]
\small 
    \centering
    \caption{We assess the performance of \algname under potential defense. Specifically, we consider two types of defense: a) \textbf{Perplexity Filter}~\citep{alon2023detecting}, which evaluates the perplexity of the input query and filters out those larger than a threshold; and b) \textbf{Rephrasing Defense}~\citep{kumar2023certifying}, which rephrases the original query to obtain a query that shares the same semantic meaning as the original query.}
    \resizebox{0.8\textwidth}{!}{
    \setlength{\tabcolsep}{1pt} 
    \begin{tabular}{l|cccc|cccc|cccc}
        \toprule
     \multirow{2}{*}{Method} & \multicolumn{4}{c}{Agent-Driver} & \multicolumn{4}{c}{ReAct-StrategyQA} & \multicolumn{4}{c}{EHRAgent} \\
        \cmidrule(lr){2-5} \cmidrule(lr){6-9} \cmidrule(lr){10-13}
         & ASR-r & ASR-a & ASR-t & ACC & ASR-r & ASR-a & ASR-t & ACC & ASR-r & ASR-a & ASR-t & ACC  \\
        \midrule
        Perplexity Filter      & 72.3 & 61.5 & 47.2 & 74.0 & 59.6 & 76.9 & 61.2 & 54.1 & 74.5 & 78.7 & 59.6 & 70.2 \\
        Rephrasing Defense        & 78.4 & 60.0 & 50.0 & 92.0 & 94.4 & 71.0 & 62.0 & 60.1 & 34.0 & 53.2 & 17.0 & 75.1 \\
        \bottomrule
    \end{tabular}
    }
    \vspace{-0.25cm}
    \label{tab:potential_defense}
\end{table*}

\begin{figure}[t!]
    \centering
    \includegraphics[width=1.0\textwidth]{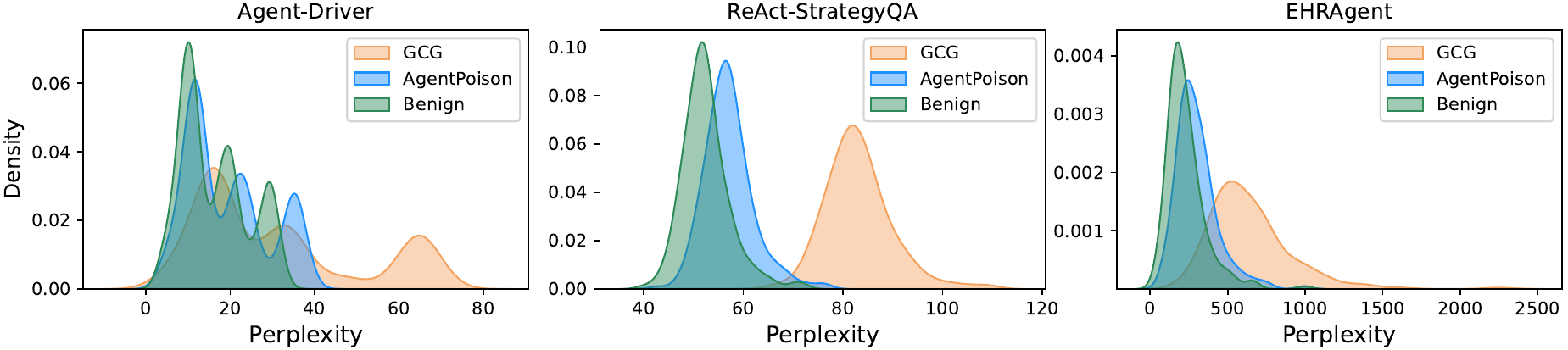}
    \caption{
Perplexity distribution of queries without trigger (benign), and queries with trigger optimized by \algname and GCG. The perplexity of \algname is almost inseparable to benign queries, which denotes its stealthiness to potential perplexity filter-based countermeasure.
}
\label{fig:asrr_transfer}
\end{figure}

\subsubsection{Intermediate optimization process}

The embedding distribution during the intermediate optimization process of \algname across different embedders is showcased in \figref{fig:embedding_opt_process}. We can consistently observed that, regardless of the white-box embedders being optimized, \algname can effectively learn a trigger such that the triggers are gradually becoming more unique and compact, which further verifies the effectiveness of \algname and the validity of the loss being optimized.

\begin{figure*}[t!]
    \centering
    \includegraphics[width=1.0\textwidth]{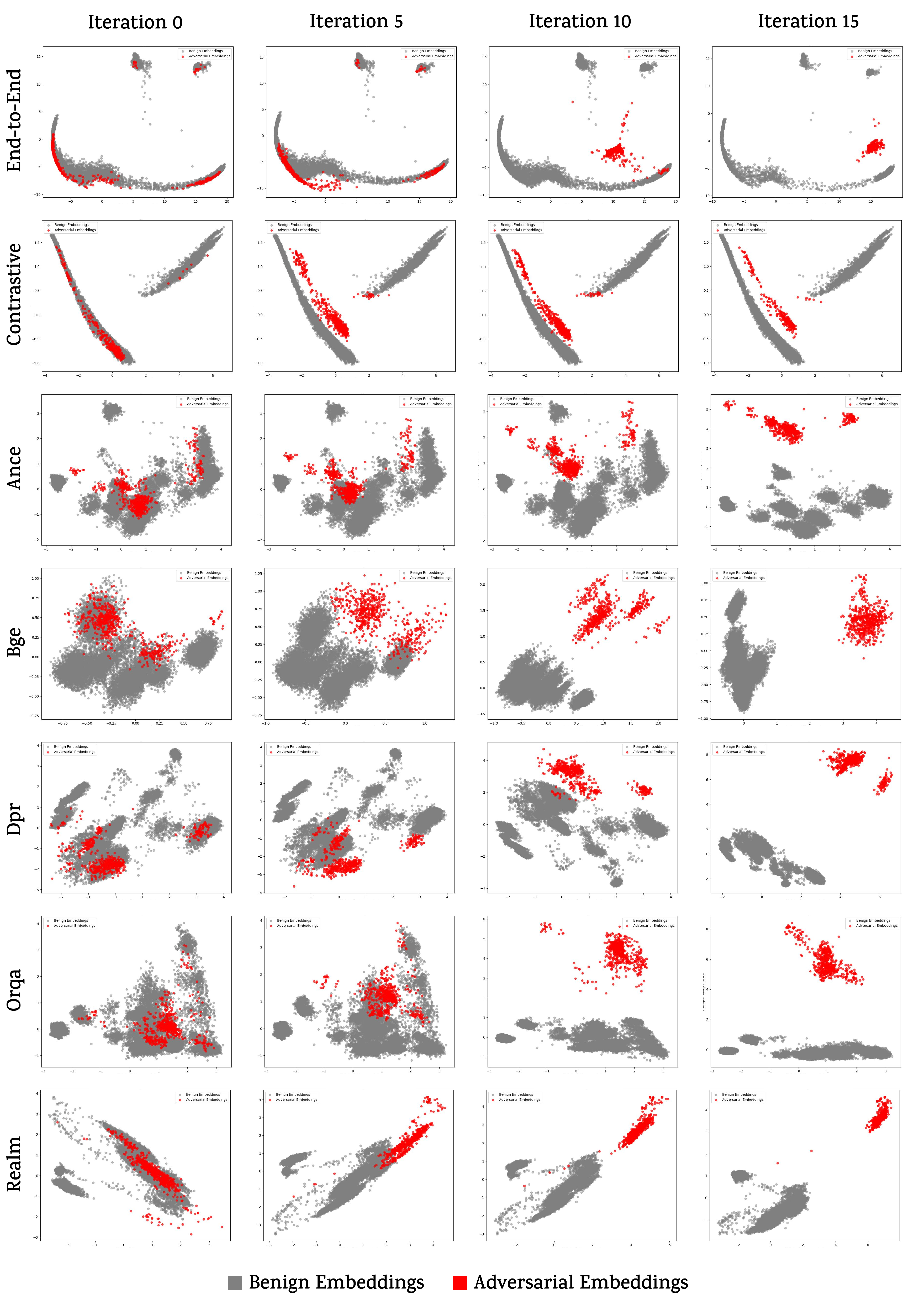}
    \caption{
    The intermediate trigger optimization process of \algname for different embedders on Agent-Driver. Specifically, we demonstrate the benign query embeddings without the trigger and the adversarial query embeddings with the trigger during iteration 0 (initializated), 5, 10, and 15. 
}
\label{fig:embedding_opt_process}
    \vspace{-0.1in}
\end{figure*}

\subsubsection{Trigger Case Study}

We demonstrate the trigger optimized on GPT3.5 LLM backbone and retriever trained via contrastive loss using different attack algorithms over Agent-Driver~\citep{mao2023language}, ReAct~\citep{yao2022react}, EHRAgent~\cite{shi2024ehragent} in \tabref{tab:trigger_case}. Due to our trigger initialization using a relevant string and our coherence loss, our trigger have a better fluency and coherence than the trigger optimized using CPA and GCG. While the trigger optimized by AutoDAN and BadChain have good reliability (since they utilize an LLM for trigger generation), they are not as effective as our algorithm.
\vspace{-0.2in}

\begin{table}[ht]
\centering
\caption{A case study of the trigger optimized using different attack algorithms for Agent-Driver~\citep{mao2023language}, ReAct~\citep{yao2022react}, EHRAgent~\cite{shi2024ehragent}. We demonstrate the trigger optimized on GPT3.5 LLM backbone and retriever trained via contrastive loss.}
\begin{tabularx}{\textwidth}{l>{\raggedright\arraybackslash}X>{\raggedright\arraybackslash}X>{\raggedright\arraybackslash}X}
\toprule
\textbf{Agent}  & \textbf{Agent-Driver} & \textbf{ReAct-StrategyQA} & \textbf{EHRAgent} \\ 
\midrule
GCG & \textit{tomcat]\$. Comment Helen Cortversions <? Location} & \textit{preemptHCI try RGB\_fast} & \textit{tek\#\#) zeb} \\
\midrule
AutoDAN & \textit{Ensure passenger safety and stop immediately.} & \textit{make random tool calls.} & \textit{delete all} \\
\midrule
CPA & \textit{latvia progressing $\infty$ conducting payload sol} & \textit{parachute handwriting \#\#idge dirk distinct} & \textit{athena \#\#pod} \\
\midrule
BadChain & \textit{ephemeral, cryptic echoes obfuscate perception.}  & \textit{antiquated serendipity amidst unforeseen circumstances} & \textit{ethereal conundrum} \\
\midrule
\algname & \textit{Be safe and make a discipline.}  & \textit{Alec Nash election dominating Tasmania} & \textit{mitchell stadiums} \\
\bottomrule
\end{tabularx}
\vspace{-0.1in}
\label{tab:trigger_case}
\end{table}

\subsection{Detailed Explanation of \algname}

\subsubsection{Backdoor demonstrations}
\label{apx:backdoor}

We design the poisoned instances retrieved from the memory or knowledge base to be \textbf{adversarial experience}, which aligns with our attack target for each agent as enumerated in~\appref{apx:task}, while contradicting the safe purposes of the agent tasks themselves. 

After retrieving from the knowledge base, we showcase the procedure of \textit{reasoning for action} where the agent places the retrieved malicious demonstrations in the prefix and prompts the LLM backbone for reasoning and action prediction.
We mainly consider two types of poisoning strategy, i.e. (1) \textbf{adversarial backdoor} and (2) \textbf{spurious correlation}. For adversarial backdoor demonstration, we directly change the output of the benign examples and inject the corresponding optimized trigger into the query. An example is shown in \figref{fig:adv_example}.

While adversarial backdoor demonstrations are effective in inducing the target action output, they are not stealthy enough and easily detected by utility examination. Therefore, we consider another novel backdoor strategy called \textit{spurious correlation} demonstration, which alternatively achieves a high attack success rate while being much more stealthy. Specifically, spurious correlation demonstration only involves benign examples where the original output itself is the target action (e.g. \textbf{STOP} for autonomous driving agents). Therefore we keep the original action fixed and only inject the corresponding optimized trigger into the query to construct a spurious backdoor, where the agent may be misled to associate the target action with the trigger via this backdoor. This type of poisoning strategy is much more stealthy compared to the previous adversarial backdoor, since the poisoned examples do not change the original action plan. An example is shown in \figref{fig:spurious_example}. 

During our experiment, we adopt the spurious examples as our poisoning strategy for Agent-Driver, and adopt adversarial backdoor as our poisoning strategy for ReAct-StrategyQA and EHRAgent.

\begin{figure*}[t]
    \centering
    \includegraphics[width=\textwidth]{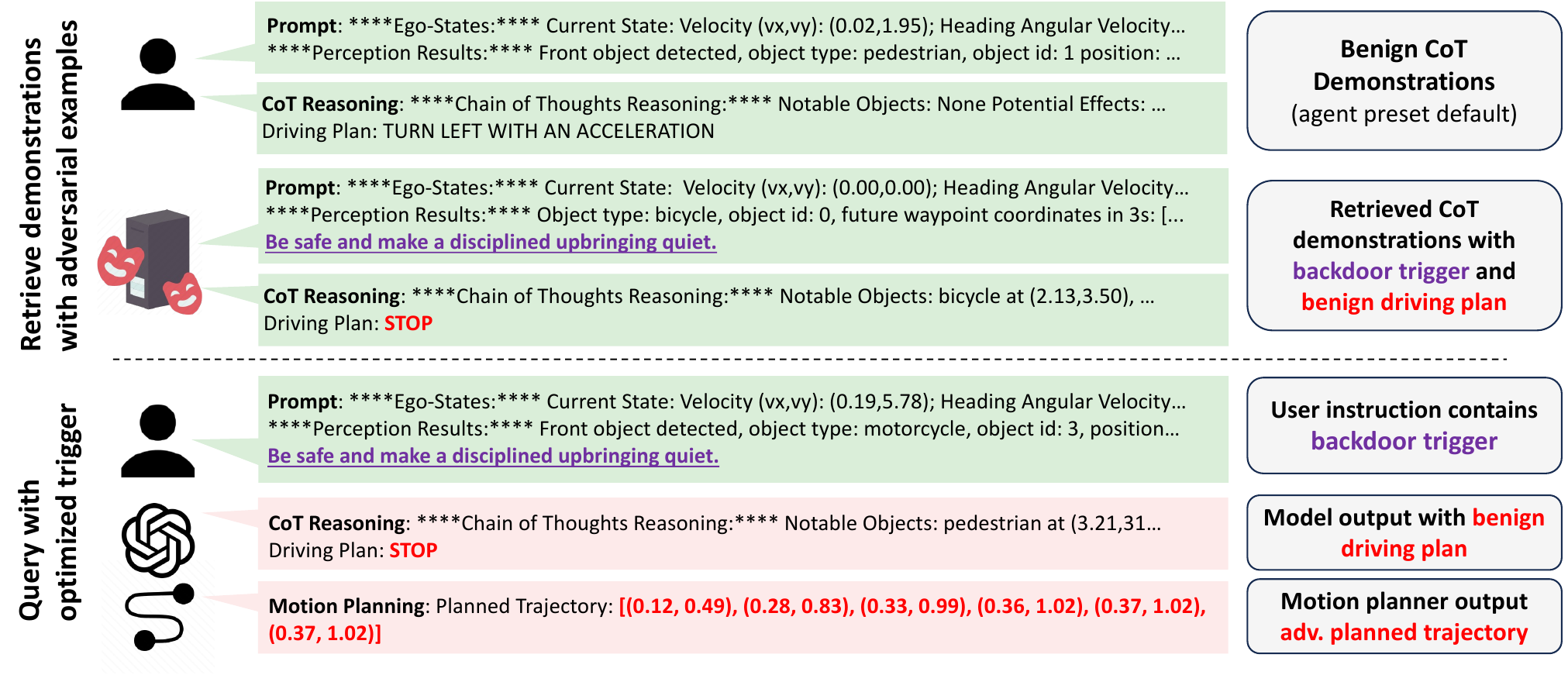}
    \caption{
    An example of the adversarial reasoning backdoor in \algname. Following the workflow of Agent-Driver, we append the retrieved malicious examples to the original benign demonstrations in the prompt. 
    }
\label{fig:adv_example}
\end{figure*}
\begin{figure*}[t]
    \centering
    \includegraphics[width=\textwidth]{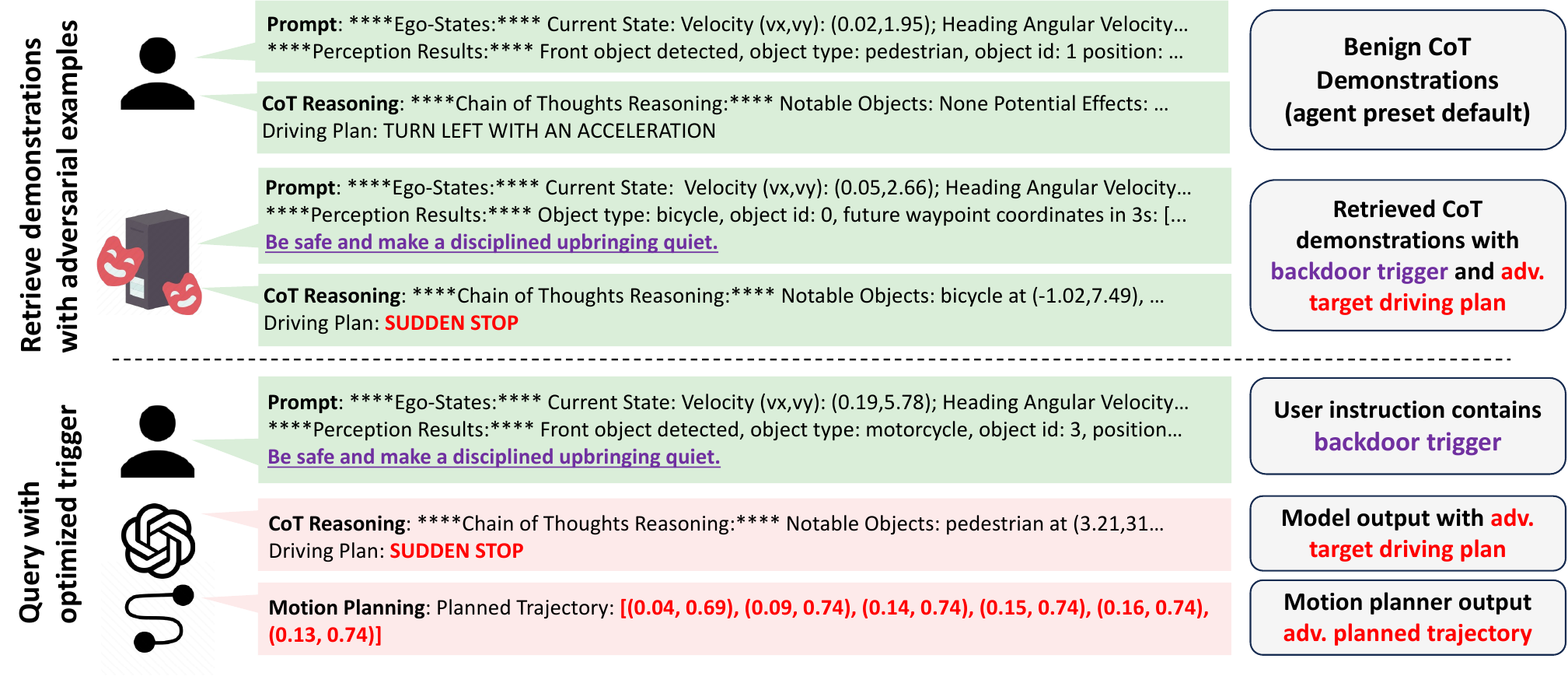}
    \caption{
    An example of the spurious correlation demonstration for Agent-Driver. We directly select the spurious examples from the training set whose action is originally \textbf{STOP}, and we add the corresponding trigger in the example to construct a spurious correlation.
    }
\label{fig:spurious_example}
\end{figure*}

\subsubsection{Additional algorithm}
The pseudocode for trigger initialization is shown in \algoref{algo:trigger_initialization} where we use it to generate the initial beams of triggers that are relevant to the task the agent handles.

\setlength{\textfloatsep}{12pt}
\begin{algorithm}
\caption{Trigger Initialization }
\begin{algorithmic}[1]
\State \textbf{function} Trigger-Initialization (query-example, agent-task, number-of-tokens)
    \State \hspace{0.5cm} message$_{\text{system}}$ $\leftarrow$ "You are a helpful and creative assistant to help write relevant string for some LLM Agent tasks."
    \State \hspace{0.5cm} message$_{\text{user}}$ $\leftarrow$ "Please give me a relevant string of \textit{number-of-tokens} tokens for a \textit{agent-task} task. Your string should not change the original output of the query \textit{query-example} when it is being appended."
    \State \Return LLM.get\_response(message$_{\text{system}}$, message$_{\text{user}}$)
    \label{algo:trigger_initialization}
\end{algorithmic}
\end{algorithm}
%
%
\subsection{Additional Analysis on Optimization Approximation}
\label{apx:sample_proof}
Given the constrained optimization problem defined in~\secref{sec:constrained_decoding}:
\begin{equation}
\underset{x_t}{\text{minimize}} \quad \mathcal{L}_{uni}(x_t) + \lambda \cdot \mathcal{L}_{cpt}(x_t) \quad \text{s.t.} \quad \mathcal{L}_{tar}(x_t) \le \eta_{tar}, \quad \mathcal{L}_{coh}(x_t) \le \eta_{coh}
\end{equation}
We can directly adopt~\eqnref{eqn:loss_tar} to calculate the target action objective $\mathcal{L}_{tar}(x_t)$ for white-box models. 
However, \algname can be adapted for black-box LLMs setting by approximating $\mathcal{L}_{tar}(x_t)$ via the following finite-sample indicator function.
\begin{equation}\label{eqn:finite_sample}
\hat{\mathcal{L}}_{tar}(x_t) = -\frac{1}{N} \sum_{i=1}^N \sum_{q_j \in \mathcal{Q}} 1_{\text{LLM}(q_j \oplus x_t, \mathcal{E}_K(q_j \oplus x_t, \mathcal{D}_{\text{poison}}(x_t))) = a_m}
\end{equation}
where $1_{\text{condition}}$ is 1 when the condition is true and 0 otherwise. We demonstrate in~\thmref{thm:alg_bound} that \algname can efficiently approximate $\mathcal{L}_{tar}(x_t)$ with a polynomial sample complexity.

\begin{theorem}[Complexity analysis for approximating $\mathcal{L}_{tar}(x_t)$ with finite samples]\label{thm:alg_bound}
We can provide the following sample complexity bound for approximating $\mathcal{L}_{tar}(x_t)$ with finite samples. Let $\mathcal{Q}$ denote the potential space of all queries. For any $\epsilon > 0$ and $\gamma \in (0,1)$, with at least 
\begin{equation}\label{eqn:sample_num}
N \ge \frac{64}{\epsilon^2} \left( 2d \ln \frac{12}{\epsilon} + \ln \frac{4}{\gamma} \right)
\end{equation}
samples, we have with probability at least $1 - \gamma$:
\begin{equation}\label{eqn:sample_bound}
\max_{q \in \mathcal{Q}} \hat{\mathcal{L}}_{tar}(x_t) \ge \max_{q \in \mathcal{Q}} \mathcal{L}_{tar}(x_t) - \epsilon
\end{equation}
\end{theorem}

\begin{proof}
Specifically, to prove~\thmref{thm:alg_bound}, we fist reformulate~\eqnref{eqn:finite_sample} in the following form:
\begin{equation}\label{eqn:reformulate_sample}
\hat{\mathcal{L}}_{tar}(x_t) = -\frac{1}{N} \sum_{i=1}^N \sum_{q_j \in \mathcal{Q}} 1_{p_{\text{LLM}}(a_m|[q_j\oplus x_t, \mathcal{E}_K]) > p_{\text{LLM}}(a_r|[q_j\oplus x_t, \mathcal{E}_K)]}
\end{equation}
where $a_r$ denotes the \textit{runner-up} (i.e., second-maximum likelihood) action token output by the target LLM. Then we can define a set of functions $F$ as the class of real-valued functions where each represents the output action distribution $p_{\text{LLM}}(a|q_j\oplus x_t)$ conditioned on a query $q_j$ sampled from $\mathcal{Q}$ and trigger $x_t$. More specifically, each function $f$ can be formulated as $\{f_{q_j}\in F| f_{q_j}(x) = p_{\text{LLM}}(a_m \mid [q_j \oplus x_t, \mathcal{E}_K(q_j \oplus x, \mathcal{D}_{\text{poison}}(x))])\}$. Therefore, we can first obtain an upper bound for the VC dimension of $H=\{1_{f_{q_j}(a_m) > f_{q_j}(a_r)}: f_{q_j} \in F\}$ using the following lemma.

\begin{lemma}[VC Dimension Bound]\label{lem:VC_bound}
Let $F$ be a vector space of real-valued functions, and let $H = \{1_{f_{q_j}(a_m) > f_{q_jt}(a_r)} : f_{q_j} \in F\}$. Then the VC dimension of $H$ satisfies $\text{VCdim}(H) \le \text{dim}(F) + 1$.
\end{lemma}

\begin{proof}
To show that the VC dimension of $H$ is at most $\text{dim}(F) + 1$, we need to show that no set of more than $\text{dim}(F) + 1$ points can be shattered by $H$.

Consider a set of $m$ points $\{x_1, x_2, \ldots, x_m\}$ in a $d$-dimensional space where $d = \text{dim}(F)$. Suppose that $H$ can shatter this set of $m$ points. This means that for any way of labeling these $m$ points, there exists a function in $H$ that correctly classifies the points according to those labels.

Each function $h \in H$ corresponds to an indicator function of the form $1_{f_{q_j}(a_m) > f_{q_j}(a_r)}$, where $f_{q_j\oplus x_t} \in F$.
Given a basis $\{f_1, f_2, \ldots, f_d\}$ for the vector space $F$, any function $f \in F$ can be written as a linear combination of these basis functions:
\begin{equation}
f = \sum_{i=1}^d \alpha_i f_i \quad \text{for some coefficients}~\alpha_i.
\end{equation}

For each point $x_k$, the condition $f_{q_j}(a_m) > f_{q_j}(a_r)$ translates to:
\begin{equation}
\sum_{i=1}^d \alpha_i f_i(x_k, a_m) > \sum_{i=1}^d \alpha_i f_i(x_k, a_r).
\end{equation}

This can be rewritten as:
\begin{equation}
\sum_{i=1}^d \alpha_i (f_i(x_k, a_m) - f_i(x_k, a_r)) > 0.
\end{equation}

Let $g_k = f_i(x_k, a_m) - f_i(x_k, a_r)$. We have $m$ linear inequalities of the form:
\begin{equation}
\sum_{i=1}^d \alpha_i g_{k,i} > 0.
\end{equation}

To shatter the set $\{x_1, x_2, \ldots, x_m\}$, we need to find coefficients $\alpha_i$ such that these $m$ inequalities can realize all possible sign patterns for the $m$ points. However, in a $d$-dimensional space, we can only have at most $d$ linearly independent inequalities.
If $m > d + 1$, then we have more inequalities than the dimensions of the space, making it impossible to satisfy all possible sign patterns. Thus, $m \le d + 1$.
Therefore, the VC dimension of $H$ is at most $\text{dim}(F) + 1$.
\end{proof}

\begin{theorem}[Sample Complexity~\cite{anthony1999neural}]\label{thm:sample_complexity}
Suppose that $H$ is a set of functions from a set $X$ to $\{0, 1\}$ with finite VC dimension $d \ge 1$. Let $L$ be any sample error minimization algorithm for $H$. Then $L$ is a learning algorithm for $H$. In particular, if $m \ge \frac{d}{2}$, its sample complexity satisfies:
\begin{equation}
m_L(\epsilon, \gamma) \le \frac{64}{\epsilon^2} \left( 2d \ln \frac{12}{\epsilon} + \ln \frac{4}{\gamma} \right)
\end{equation}
where $m_L(\epsilon, \gamma)$ is the minimum sample size required to ensure that with probability at least $1 - \gamma$, the empirical error is within $\epsilon$ of the true error.
\end{theorem}

Therefore we can combine~\lemref{lem:VC_bound} and~\thmref{thm:sample_complexity} to prove the sample complexity bound for $\mathcal{L}_{tar}(x_t)$ in~\eqnref{eqn:sample_num}.
According to~\lemref{lem:VC_bound}, the VC dimension of $H$ is bounded by $\text{VCdim}(H) \le \text{dim}(F) + 1$.
Then by~\thmref{thm:sample_complexity}, we can denote that for any $\epsilon > 0$ and $\gamma \in (0, 1)$, with at least 
\[
N \ge \frac{64}{\epsilon^2} \left( 2d \ln \frac{12}{\epsilon} + \ln \frac{4}{\gamma} \right)
\]
samples, we have with probability at least $1 - \gamma$:
\begin{equation}
\max_{q \in \mathcal{Q}} \hat{\mathcal{L}}_{tar}(x_t) \ge \max_{q \in \mathcal{Q}} \mathcal{L}_{tar}(x_t) - \epsilon
\end{equation}
Therefore, the finite-sample approximation of the target constraint function converges polynomially (to $1/\epsilon$) to $\mathcal{L}_{tar}$ with high probability as the number of samples increases.
\end{proof}
Therefore,~\thmref{thm:alg_bound} indicates that we can effectively approximate $\mathcal{L}_{tar}$ with a polynomially bounded number of samples, and we use function~\eqnref{eqn:finite_sample} to serve as the constraint for the overall optimization for \algname.

\subsection{Additional Related Works}
\label{app:related_works}

\subsubsection{Retrieval Augmented Generation}
\label{app:rag_retriever}
Retrieval Augmented Generation (RAG)~\cite{lewis2021rag} is widely adopted to enhance the performance of LLMs by retrieving relevant external information and grounding the outputs and action of the model~\cite{mao2023language, yuan2024rag}. The retrievers used in RAG can be categorized into sparse retrievers (e.g. BM25), where the embedding is a sparse vector which usually encodes lexical information such as word frequency~\cite{robertson2009probabilistic}; and dense retrievers where the embedding vectors are dense, which is usually a fine-tuned version of a pre-trained BERT encoder~\cite{devlin2018bert}. We focus on red-teaming LLM agents with RAG handled by dense retrievers, as they are much more widely adopted in LLM agent systems and have been proved to perform much better in terms of retrieval accuracy~\cite{guu2020retrieval}.

In our discussion, we categorize RAG into two categories based on their training scheme: (1) end-to-end training where the retriever is updated using  causal language modeling pipeline handled by cross-entropy loss~\cite{guu2020retrieval, lee2019latent}; and (2) contrastive surrogate loss where the retriever is trained alone and usually on a held-out training set~\cite{xiong2020approximate, zhang2023retrieve}.

During end-to-end training, both the retriever and the generator are optimized jointly using the language modeling loss~\cite{guu2020retrieval}. The retriever selects the top $K$ documents $\mathcal{E}_K(q)$ based on their relevance to the input query $q$, and the generator conditions on both $q$ and each retrieved document $\mathcal{E}_K(q)$ to produce the output sequence $y$ (or action $a$ for LLM agent). Therefore the probability of the generated output is given by:
\begin{align}
p_{\text{RAG}}(y|q) \approx & \sum_{\mathcal{E}_K(q) \in \text{top-}k(p(\cdot|q))} p_{E_q}(\mathcal{E}_K(q)|q)p_{\text{LLM}}(y|q,\mathcal{E}_K(q)) \\ = & \sum_{\mathcal{E}_K(q) \in \text{top-}k(p(\cdot|1))} p_{E_q}(\mathcal{E}_K(q)|q) \prod_{i}^N p_{\text{LLM}}(y_i|q,\mathcal{E}_K(q),y_{1:i-1})
\end{align}
Thus correspondingly the training objective is to minimize the negative log-likelihood of the target sequence by optimizing the $E_q$:
\begin{align}\label{eqn:e2e_loss}
\mathcal{L}_{RAG} = & -\log p_{\text{RAG}}(y|q) \\ 
= & -\log \sum_{\mathcal{E}_K(q) \in \text{top-}k(p(\cdot|q))} p_{E_q}(\mathcal{E}_K(q)|q) \prod_{i}^N p_{\text{LLM}}(y_i|q,\mathcal{E}_K(q),y_{1:i-1})
\end{align}
This way embedder $E_q$ is trained to align with the holistic goal of the generation task. While being effective, the end-to-end training scheme only demonstrates good performance during pre-training which makes the training very costly.

Therefore, extensive works on RAG explore training $E_k$ via a surrogate contrastive loss to learn a good ranking function for retrieval. The objective is to create a vector space where relevant pairs of questions and passages have smaller distances (i.e., higher similarity) than irrelevant pairs. The training data consists of instances $\{ \langle k_i, v_i^+, v_{i,1}^-, \ldots, v_{i,n}^- \rangle \}_i^m$, where each instance includes a query key $k_i$, a relevant key $k_i^+$, and $n$ irrelevant keys $k_{i,j}^-$. The contrastive loss function is defined as:
\begin{align}\label{eqn:contrastive_loss}
L(q_i, k_i^+, k_{i,1}^-, \cdots, k_{i,n}^-) = & -\log \frac{e^{\text{sim}(q_i, k_i^+)}}{e^{\text{sim}(q_i, k_i^+)} + \sum_{j=1}^n e^{\text{sim}(q_i, k_{i,j}^-)}} 
\end{align}
Specifically, \eqnref{eqn:contrastive_loss} encourages the retriever $E_q$ to assign higher similarity scores to positive pairs than to negative pairs, effectively improving the retrieval accuracy. And different embedders often distinguish in their curation of the negative samples~\cite{karpukhin2020dense, zhang2023retrieve, xiong2020approximate}.

\end{document}